\documentclass[letterpaper,twocolumn,10pt]{article}
\usepackage{usenix}
\usepackage{mathtools}
\usepackage{amsmath,amsthm,amsfonts, amssymb}
\usepackage{algorithm}
\usepackage{lipsum}
\usepackage{subcaption}
\usepackage{booktabs}
\usepackage{multirow}

\usepackage{amsmath,amsfonts, amssymb}
\usepackage{algorithmic}
\usepackage{algorithm}

\newtheorem{theorem}{Theorem}
\newtheorem{lemma}{Lemma}
\newtheorem{corollary}{Corollary}[theorem]
\newtheorem{definition}{Definition}

\newtheorem{remark}{Remark}

\DeclareMathOperator*{\argmin}{argmin}

\newcommand{\hZ}{{\hat{Z}}}

\newcommand{\bE}{\mathbb{E}}
\newcommand{\bR}{\mathbb{R}}
\newcommand{\cL}{\mathcal{L}}
\newcommand{\cZ}{\mathcal{Z}}
\newcommand{\cM}{\mathcal{M}}

\newcommand{\bphi}{\boldsymbol{\phi}}
\newcommand{\btheta}{\boldsymbol{\theta}}

\begin{document}

\date{}


\title{\Large \bf Privacy for Fairness: Information Obfuscation for Fair Representation Learning with Local Differential Privacy }

\author{
{\rm Songjie Xie}\\
HKUST
\and
{\rm Youlong Wu\footnotemark[1]}\\
ShanghaiTech University
\and
{\rm Jiaxuan Li}\\
ShanghaiTech University
\and
{\rm Ming Ding}\\
Data61, CSIRO
\and
{\rm Khaled~B.~Letaief\footnotemark[1]}\\
HKUST
}
\maketitle
\def\thefootnote{*}\footnotetext{Corresponding authors}
\thispagestyle{empty}

\subsection*{Abstract}
As machine learning (ML) becomes more prevalent in human-centric applications, there is a growing emphasis on algorithmic fairness and privacy protection. While previous research has explored these areas as separate objectives, there is a growing recognition of the complex relationship between privacy and fairness. However, previous works have primarily focused on examining the interplay between privacy and fairness through empirical investigations, with limited attention given to theoretical exploration. This study aims to bridge this gap by introducing a theoretical framework that enables a comprehensive examination of their interrelation. We shall develop and analyze an information bottleneck (IB) based information obfuscation method with local differential privacy (LDP) for fair representation learning. In contrast to many empirical studies on fairness in ML, we show that the incorporation of LDP randomizers during the encoding process can enhance the fairness of the learned representation. Our analysis will demonstrate that the disclosure of sensitive information is constrained by the privacy budget of the LDP randomizer, thereby enabling the optimization process within the IB framework to effectively suppress sensitive information while preserving the desired utility through obfuscation. Based on the proposed method, we further develop a variational representation encoding approach that simultaneously achieves fairness and LDP. Our variational encoding approach offers practical advantages. It is trained using a non-adversarial method and does not require the introduction of any variational prior. Extensive experiments will be presented to validate our theoretical results and demonstrate the ability of our proposed approach to achieve both LDP and fairness while preserving adequate utility.

\section{Introduction}


With the increasing accessibility of computational resources and vast datasets, machine learning (ML) has become a crucial framework for numerous human-centric applications. However, beyond improving the accuracy of ML algorithms, their successful deployment in real-world scenarios necessitates the establishment of reliable and trustworthy ML systems, which entails meeting additional criteria beyond performance alone. One critical criterion is the mitigation of biases against socially sensitive groups, including gender, race, or sexual orientation. This awareness of algorithmic fairness is growing in prominence as a result~\cite{rosenberg2023fairness, feldman2015certifying, kleinberg2018algorithmic, mehrabi2021survey}.
Besides, the utilization of personal data in ML models has raised concerns regarding privacy breaches, necessitating the need for privacy protection as a prerequisite before deployment~\cite{dwork2011firm}. Fairness and privacy represent two pivotal ethical considerations in ML, which are not only integral from an ethical standpoint but are also mandated by contemporary and forthcoming governmental regulations such as the General Data Protection and Regulation (GDPR)~\cite{voigt2017eu} and the AI Act~\cite{veale2021demystifying}.

Fairness in ML encompasses a wide spectrum of research endeavors that aim at mitigating biases and discriminatory outcomes in algorithmic decision-making systems. It can be broadly classified into two main categories: Individual fairness and group fairness\footnote{The concept of subgroup fairness extends the notion of group fairness by considering fairness within subgroups of a population.}~\cite{mehrabi2021survey}. Individual fairness seeks to avoid arbitrary differentiations between practically indistinguishable individuals. Group fairness, as a widely employed concept, centers around the objective of preventing decisions from unfairly favoring or disadvantaging specific groups based on sensitive attributes. Recently, numerous endeavors have been made to achieve algorithmic fairness, including in-processing methods that incorporate fairness constraints~\cite{kamishima2011fairness, zafar2017fairness}, as well as pre/post-processing methods that modify the representations of training data and/or address potentially unfair predictions~\cite{dwork2012fairness,alghamdi2022beyond, calmon2017optimized, louizos2015variational, zhao2019conditional,madras2018learning, bertran2019adversarially}.
On the other hand, differential privacy has emerged as a widely recognized privacy framework for safeguarding the opt-out rights of individuals~\cite{dwork2006calibrating, dwork2011firm}. In essence, differential privacy ensures that attackers cannot exploit the results of queries or released models to discern the presence or absence of specific records in the underlying dataset.

Across various scenarios, fairness and privacy are deeply intertwined and cannot be viewed as independent objectives in ML.
Seminal studies have demonstrated the inherent connection between certain notions of fairness and privacy. For example, group fairness is equivalent to privatizing group membership attributes~\cite{mozannar2020fair}, while individual fairness aligns with differential privacy~\cite{dwork2012fairness}.
Moreover, recent empirical studies have shown that the direct implementation of private mechanisms can result in disparate performance across different demographic groups~\cite{fioretto2022differential, bagdasaryan2019differential, pujol2020fair, farrand2020neither}, which impairs fairness. For instance, models trained with DP-SGD~\cite{abadi2016deep} exhibit a larger accuracy reduction on unprivileged subgroups~\cite{bagdasaryan2019differential}, and privacy-preserving data can be used to infer sensitive attributes and introduce unintended biases into decision-making~\cite{pujol2020fair}.
However, despite the existing empirical studies that have imposed the contrasting goals of privacy and fairness, there remains a lack of theoretical analysis examining their interrelation. 
Consequently, the development of fairness-aware learning algorithms that effectively achieve privacy preservation remains largely unexplored in the field of trustworthy ML, especially in the context of preprocessing methods for fairness and differential privacy.

\begin{figure}[t]
		\centering
		\includegraphics[width=0.98\linewidth]{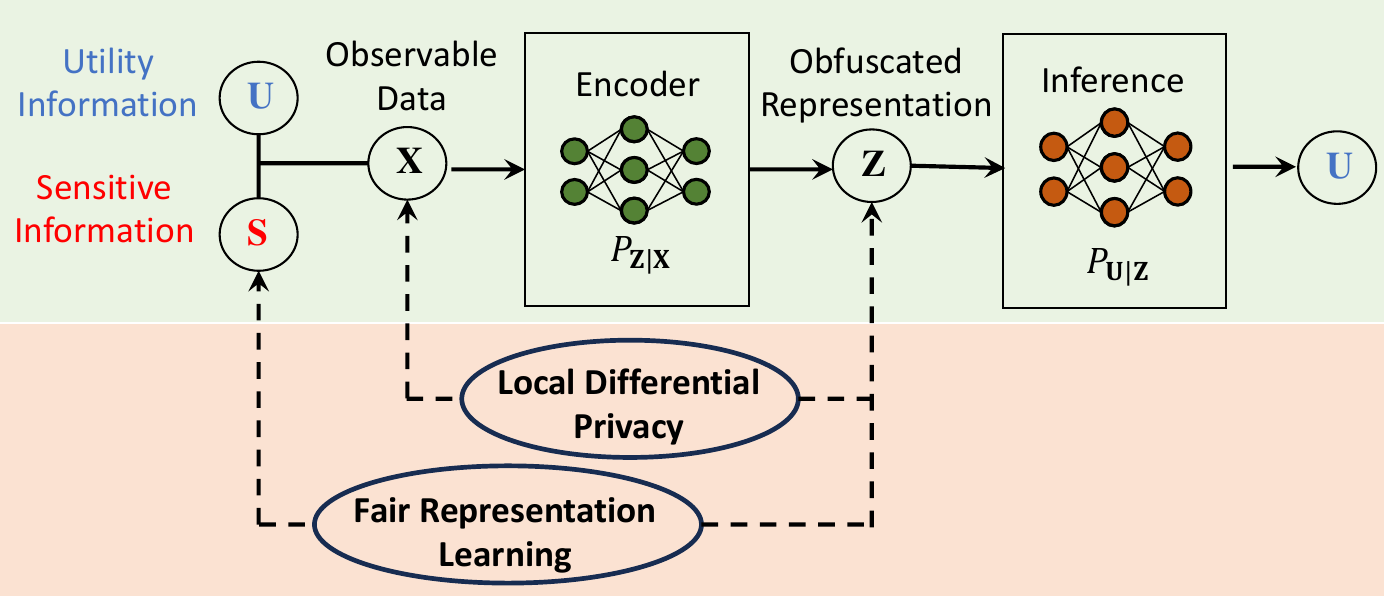}
		\caption{ The considered scenario of fair representation learning for information obfuscation with local differential privacy.}
		\label{fig:illu-model}
	\end{figure}

To address \textcolor{blue}{this} challenge, we adopt a theoretical approach to characterize the interrelation, particularly focusing on fair representation learning with local differential privacy (LDP), as illustrated in Figure \ref{fig:illu-model}. Fair representation learning~\cite{zemel2013learning} is a preprocessing fair ML approach that encodes the data for utility purposes while simultaneously obfuscating any information about membership of sensitive groups to prevent potential discrimination in unknown downstream tasks. By incorporating LDP~\cite{cormode2018privacy} into the encoding process, we can ensure that the representations released from the data owner are indistinguishable, thereby guaranteeing the privacy of the raw data. Our research addresses the problem of creating a trustworthy pre-processing step for various practical ML applications, such as making hiring decisions, approving loans, and analyzing health data.

In this work, we utilize an information bottleneck framework (IB, \cite{tishby2000information, du2012privacy, fischer2020conditional, rodriguez2021variational, bertran2019adversarially}) where mutual information (MI) acts as both a measure of utility and obfuscation. Unlike existing IB-based methods that rely on adversarial techniques to directly reduce the MI between sensitive variables and representations, we leverage the LDP mechanism to limit sensitive information leakage. By doing so, we convert the MI minimization problem into a more manageable maximization problem. Our approach is non-adversarial and is implemented using variational techniques, without requiring a variational prior.
 
Our contributions are summarized as follows: 
\begin{enumerate}
    \item To the best of our knowledge, \textcolor{blue}{this} is the first work to introduce LDP into fair representation learning. We propose a novel IB-based framework with LDP guarantee for sensitive information obfuscation. Based on the proposed IB-based framework, we propose a variational fair representation encoding approach that is non-adversarial and does not require a variational prior.
    \item We provide a comprehensive theoretical analysis of the relationship between fairness and privacy. Our analysis demonstrates that the use of LDP randomizers can enhance fairness within the proposed Information Bottleneck (IB) framework. Specifically, the privacy budget associated with LDP randomizers sets an upper limit on information leakage, thereby enabling the optimization process within the IB framework to effectively mitigate the disclosure of sensitive information.
    \item We conduct experiments on real-world datasets to validate our theoretical results and demonstrate that the proposed variational approach achieves LDP and fairness while maintaining sufficient utility.
\end{enumerate}

The rest of the paper is organized as follows. We formally formulate the problem and present the preliminaries in Section~\ref{sec:problem}. In Section~\ref{sec:method-1}, we present the proposed method of information obfuscation with LDP and provide the corresponding theoretical analysis, as well as the main result. Section~\ref{sec: variational} proposes a variational encoding method. Section~\ref{sec: exp} will present extensive numerical results, which will validate our main results and the performance of the proposed variational encoding method in Section~\ref{sec: exp} followed by discussions and conclusion in Section~\ref{sec: dis} and Section~\ref{sec: con}, respectively.

\section{Problem Formulation and Preliminaries}
\label{sec:problem}

\paragraph{Notations.} \emph{Throughout this paper, we denote random variables by capital letters (e.g., $X$) and their samples by lowercase letters (e.g., $x$). The distribution of a random variable $X$ is denoted by $P_X$ and its probability density (or mass) function by $p_X(x)$. We may drop the capital letter when it is clear from the context (e.g., $p_X(x) = p(x)$), and use a subscript to emphasize the dependence of measures on the choice of distribution parameterization (e.g., $p_{\phi}(\hat{z}|x)$). The expectation is denoted by $\bE[\cdot]$. The Shannon entropy and mutual information are denoted by $H(X)$ and $I(X;Y)$, respectively. $[K]$ shall denote the set $\{1,2, \dots, K\}$ for the positive integer $K\geq 1$.}

We consider a scenario in which the data owner has access to three random variables, $U$, $S$, and $X$, with a joint distribution $P_{USX}$. The variable $X\in\mathcal{X}$ is the observable data variable, $U\in\mathcal{U}$ is the utility variable utilized for the downstream task, and $S\in\mathcal{S}$ is the sensitive variable related to the sensitive group membership of each data sample (e.g., disability, ethnicity, and sexual orientation). In the considered problem, the observable data $X$ is encoded into the representation $Z \in \mathcal{Z}$, which can be formally described with the help of the following Markov chain
\begin{align}
    (U, S) \longleftrightarrow X \longleftrightarrow Z. \label{eq: markov-1}
\end{align}
Our goal is to encode each data point $x \sim X$ into a representation $z \sim {Z}$ that preserves the utility about $U$ while ensuring two objectives: \romannumeral1) \emph{local differential privacy}; and \romannumeral2) \emph{fairness}. We elaborate on these objectives in Section~\ref{subsec:ldp} and Section~\ref{subsec:ULT}, respectively.

\subsection{Local Differential Privacy}\label{subsec:ldp}
Local differential privacy (LDP)~\cite{cormode2018privacy} is a stronger form of differential privacy that guarantees that a randomized mapping generates indistinguishable distributions over outputs on any different data samples, without relying on a trustworthy central authority. Let $P_{Z|X}$ be a mapping from data source $X$ to encoded representation $Z$. LDP ensures that \emph{$Z$ should be indistinguishable such that any adversary cannot distinguish any different data samples $x\not= x' \in \mathcal{X}$ based on released representations}, which is formally defined as follows:
\begin{definition}[Local differential privacy]
    For any $\epsilon \geq 0$, a mapping $P_{Z|X}: \mathcal{X} \to \mathcal{Z}$ is $\varepsilon$-local differential private, if for every $x \not = x' \in \mathcal{X}$ and any measurable set $C \in \mathcal{Z}$ we have 
    \begin{align}
        \frac{\Pr[z\in C | x]}{\Pr[z\in C |x']} \leq e^{\epsilon}. \label{eq: ldp}
    \end{align}
\end{definition}

\subsection{Fairness: Utility-Leakage Tradeoff}\label{subsec:ULT}
To ensure the fairness of learned representations, it is crucial to \emph{obfuscate any information related to the sensitive variable $S$ to avoid any potential discrimination}. Our particular emphasis lies on group fairness and fairness defined through \emph{unawareness}, which means an algorithm is fair as long as any protected attributes are not explicitly used in the decision-making process~\cite{mehrabi2021survey}. For example, if $S$ is gender and the learned representations are independent of the gender attribute of the data source, then the representations cannot introduce undesired biases in downstream models due to gender discrimination.  In this work, mutual information $I(S;Z)$ is used to measure the quantity of sensitive information contained in $Z$ with respect to $S$.

However, in most cases, $U$ and $S$ exhibit dependence, which leads to an increase in sensitive information leakage as the utility is amplified. To describe this fundamental tradeoff, we introduce the concept of a \emph{utility-leakage pair}.
\begin{definition}[Utility-leakage pair]
   For a given data source $P_{U S X}$ and a mapping $P_{Z|X}$, the utility-leakage pair $(\Gamma, \Omega)$ is defined by $\Gamma \triangleq I(U; Z)$ and $\Omega \triangleq I(S; Z)$, respectively.
\end{definition}
When considering a data source $P_{U S X}$ and a mapping $P_{Z|X}$, it is imperative to assess whether the obtained utility-leakage pairs $(\Gamma, \Omega)$ are optimal. Specifically, optimality is achieved when it is impossible to increase the utility without augmenting the leakage, or when it is impossible to reduce the leakage without diminishing the utility.
Thus, the optimality of $(\Gamma, \Omega)$ can be defined as follows.
\begin{definition}[Optimality]
    Given the optimal utility-leakage pairs are characterized by the optimization problem $G(\gamma, P_{USX}$) for $\gamma \geq 0$, defined as
    \begin{align}
        G(\gamma, P_{USX}) = \inf\limits_{P_{Z|X}: I(U; Z) \geq \gamma } I(S; Z) \label{eq:optimal}
    \end{align}
\end{definition}
The $\gamma$ in $G(\gamma, P_{USX})$ is the minimal amount of required information about $U$ disclosed via $Z$. It is noteworthy that the optimization problem presented in \eqref{eq:optimal} is a widely studied generalized IB problem that generalizes two well-known problems: Vanilla IB problems~\cite{tishby2000information} when $X=S$, and privacy funnel problems~\cite{du2012privacy} when $U=X$. In particular, the investigation of the solution $G(\gamma, P_{USX}) = 0$ for some $\gamma \geq 0$, which corresponds to achieving perfect obfuscation, has been the subject of extensive research in information theory~\cite{razeghi2020perfect, calmon2015fundamental, rassouli2021perfect}.

For the considered problem presented in \eqref{eq: markov-1}, this work aims to address the following question: \emph{How can we develop a representation encoding that achieves both privacy and fairness objectives simultaneously? Additionally, what is the interplay between privacy and fairness in this context?}

\section{Main Framework and Theoretical Analysis}
In this section, we present an IB-based framework that incorporates LDP randomizers for effective information obfuscation. We then provide a comprehensive theoretical analysis, accompanied by our main result regarding the interrelations between fairness and privacy.
\label{sec:method-1}
\subsection{Information Obfuscation with Local Differential Privacy}
Many IB-based methods for information obfuscation are subjected to a constraint on the information complexity of representation measured by $I(X; Z)$. This complexity constraint is not only driven by compression purposes but also serves to prevent sensitive information leakage~\cite{atashin2021variational}.

Instead of directly minimizing the mutual information $I(X;Z)$ through optimization,  we incorporate LDP into the encoding process and leverage the introduced privacy-preserving mechanism to limit sensitive information leakage. Specifically, we encode $X$ into an intermediate representation $\hZ$ with an encoder $p_{\bphi}(\hat{z}|x)$ that is parameterized by $\bphi$. Then, $\hat{Z}$ is mapped to the obfuscated representation $Z$ via an $\epsilon$-LDP randomizer $\cM: \mathcal{\hZ} \to \cZ$, inducing a noisy encoder $p_{\bphi}(z|x)$ from $X$ to $Z$. This probabilistic model defines the underlying data structure of our proposed method:
\begin{align}
     (U, S) \longleftrightarrow X \stackrel{\bphi}{\longleftrightarrow} \hZ  \stackrel{\cM}{\longleftrightarrow} Z.  \label{eq: markov-2}
\end{align}
  
Having introduced an $\epsilon$-LDP randomizer $\cM$, the proposed optimization for the $p_{\bphi}(\hat{z}|x)$ is:
\begin{align}
    g(\gamma, P_{USX}, \cM) = \max\limits_{\substack{p_{\bphi}(\hat{z}|x):\\Z = \cM(\hZ) \\ I(U; Z) \geq \gamma} } I(X; Z|S). \label{eq: method}
\end{align}

The proposed optimization problem $g(\gamma, P_{USX}, \cM)$ deviates from $G(\gamma, P_{USX})$ presented in \eqref{eq:optimal} in two primary aspects: Firstly, compared to $G(\gamma, P_{USX})$, it incorporates a randomizer $\cM$ in the encoding process, which introduces stochasticity into the representation $Z$ and hence limits the complexity $I(X;Z)$. Secondly, the objective function $I(X;Z|S)$ is the conditional MI that measures the amount of information retained in the representation $Z$ about $U$ while excluding $S$. From the Markov chain \eqref{eq: markov-2} and the chain rule of MI, $I(S;Z) = I(X; Z) - I(X; Z|S)$, maximizing $I(X;Z|S)$ can reduce the sensitive information leakage $I(S;Z)$ when $I(X;Z)$ is limited. 

In the following subsections, the theoretical analysis is conducted to justify the proposed design, i.e., achieving LDP and fairness simultaneously, and investigate their interplay in the proposed IB-based framework.

\subsection{Impact of LDP-Mechanisms on Representation Encoding}
In order to comprehensively examine the LDP mechanisms incorporated within the proposed methods, as delineated by \eqref{eq: method}, our objective is to characterize them as an integral component of the overall noisy encoding process $P(Z|X)$. 
We present the following lemmas to formally establish that the incorporation of an $\epsilon$-LDP mechanism $\cM$ subsequent to the encoding process $P(\hat{Z}|X)$ ensures that the encoding $P(Z|X)$ also adheres to $\epsilon$-LDP (Lemma~\ref{lem:LDP}), while simultaneously bounding the information complexity $I(Z;X)$ by the privacy budget $\epsilon$ (Lemma~\ref{lem: MIBound}).
\begin{lemma}
    For any $P_{\hat{Z}|X}$, an $\epsilon$-LDP mechanism $\cM: \mathcal{\hat{Z}} \to \mathcal{Z}$ induces a mapping $P_{Z|X}$ that satisfies $\epsilon$-LDP. \label{lem:LDP}
\end{lemma}
\begin{proof}
    We prove the lemma under the assumption that $Z$ is a discrete variable. The proof can be easily extended to the continuous case and we leave this extension to the interested reader. For any measurable set $C \in \mathcal{Z}$, we use $p_{\text{min}, C}$ to denote the minimal conditional distribution $p_{\text{min}, C} = \min_{\hat{z}\in \mathcal{\hat{Z}}} \Pr[z\in C|\hat{z}]$. Then, for each $x \not = x' \in \mathcal{X}$, we have 
    \begin{align}
        \Pr[z\in C | x] &= \sum\limits_{\hat{z} \in \mathcal{\hat{Z}}} \Pr[z\in C | \hat{z}] p_{\hat{Z}|X}(\hat{z} |x)\\ 
        & \overset{(a)}{\leq} \sum\limits_{\hat{z}\in \mathcal{\hat{Z}}} e^{\epsilon} p_{\text{min}, C} p_{\hat{Z}|X}(\hat{z}|x)\\
        & = e^{\epsilon} p_{\text{min}, C}\\
        & = e^{\epsilon} \sum\limits_{\hat{z}\in \mathcal{\hat{Z}}} p_{\text{min}, C} p_{\hat{Z}|X}(\hat{z}|x')\\
        & \overset{(b)}{\leq} e^{\epsilon}\sum\limits_{\hat{z}\in \mathcal{\hat{Z}}}  \Pr[z\in C|\hat{z}]p_{\hat{Z}|X}(\hat{z}|x')\\
        & = e^{\epsilon} \Pr[z\in C | x'], \label{eq: proof-lemma-1}
    \end{align}
    where the first inequality $(a)$ follows from the definition of $\epsilon$-LDP randomizer $\cM$ and the second inequality $(b)$ holds because $\Pr[z\in C|\hat{z}] \geq \min_{\hat{z}\in \mathcal{\hat{Z}}} \Pr[z\in C|\hat{z}]=p_{\text{min}, C}$ for each $\hat{z}\in \mathcal{\hat{Z}}$. Hence, the induced mapping $P_{Z|X}$ satisfies $\epsilon$-LDP from \eqref{eq: proof-lemma-1}. 
\end{proof}
With Lemma~\ref{lem:LDP}, we can ensure that, once the $\epsilon$-LDP mechanisms are integrated, no matter what the mapping $P_{\hat{Z}|X}$ is learned, the induced mapping $P_{Z|X}$ also satisfies $\epsilon$-LDP. This lemma ensures that the representation encoding process is $\epsilon$-LDP all the time, even in the training process.

Then, we prove that introducing $\cM$ can limit the information complexity $I(X;Z)$ by the privacy budget $\epsilon$, formally stated in the following lemma.
\begin{lemma}
An $\epsilon$-LDP mechanism $P_{Z|X}$ generates a representation $Z = \cM(\hat{Z})$ with 
\begin{align}
    I(Z; X) &\leq \epsilon.
\end{align}
\label{lem: MIBound}
\end{lemma}
\begin{proof}
By the definition of $\epsilon$-LDP in \eqref{eq: ldp}, for each $z \in \mathcal{Z}$ and $x \in \mathcal{X}$, we have
\begin{align}
    p_Z(z)  &= \bE_{p_X(x')}\left[p_{Z|X}(z|x')\right] \\
    &\geq \bE_{p_X(x')}\left[p_{Z|X}(z|x)e^{-\epsilon}\right]\\
    &= p_{Z|X}(z|x)e^{-\epsilon}. \label{eq: proof-lemma-2-bounded }
\end{align}
This means that the marginal distribution $p_Z(z)$ of each $z\in\mathcal{Z}$ is lower bounded by its conditional distribution $p_{Z|X}(z|x)e^{-\epsilon}$. With this result, we can bound the mutual information $I(X; Z)$ as follows:
\begin{align}
    I(Z; X) & = \bE_{p_{ZX}(z,x)}\left[\log \frac{p_{ZX}(z,x)}{p_Z(z)p_X(x)}\right]\\
    & = \bE_{p_{ZX}(z,x)}\left[\log \frac{p_{Z|X}(z|x)}{p_{Z}(z)}\right]\\
    & \overset{(a)}{\leq} \bE_{p_{ZX}(z,x)}\left[\log \frac{p_{Z|X}(z|x)}{p_{Z|X}(z|x)e^{-\epsilon}}\right]\\
    & = \epsilon,
\end{align}
where the inequality $(a)$ follows from the inequality \eqref{eq: proof-lemma-2-bounded }.
\end{proof}
Lemma~\ref{lem: MIBound} states that as long as $\cM$ is added to the encoding process $P_{Z|X}$, the information complexity $I(X; Z)$ is bounded by the privacy budget $\epsilon$. It is a natural result of applying $\epsilon$-LDP randomizers. Nevertheless, this property can be effectively exploited in our proposed optimization Problem~\eqref{eq: method} as it ensures that the information complexity $I(X;Z)$ is limited by $\epsilon$.

\subsection{Theoretical Relations and Guarantees on Privacy and Fairness}
Having investigated the impacts of the incorporation of LDP randomizers, we now present our main result on the interrelation between privacy and fairness. 
Specifically, the following theorem demonstrates that the obtained mapping $P_{Z|X}$ can achieve LDP and fairness simultaneously, with the two objectives mutually reinforcing each other.  
\begin{theorem}
Given a data source $P_{USX}$ and an $\epsilon$-LDP randomizer $\cM: \mathcal{\hat{Z}} \to \mathcal{Z}$, the encoder $p_{\bphi}(\hat{z}|x)$ obtained by $\nu^* = g(\gamma, P_{USX}, \cM)$ can induce an $\epsilon$-LDP mapping $P_{Z|X}$ that achieves an utility-leakage pair $(\Gamma, \Omega)$ with $\gamma \leq \Gamma \leq \epsilon$ and $\Omega \leq \epsilon - \nu^*$.
\label{th: main}
\end{theorem}
\begin{proof}
Firstly, by Lemma \ref{lem:LDP}, an $\epsilon$-LDP randomizer $\cM$ can induce an $\epsilon$-LDP mapping $P_{Z|X}$ for any $P_{\hat{Z}|X}$.

Next, we consider the sensitive information leakage $\Omega$. By the Markov chain $(U,S) \to X \to Z$ and the chain rule of MI, we have
\begin{align}
\Omega &= I(S;Z) \\
&= I(X;Z) - I(X;Z|S).
\end{align}
We want to upper bound $\Omega$ in terms of $\epsilon$.
To this end, suppose there exists a solution $\nu^*$ of the optimization Problem \eqref{eq: method} that achieves $I(X;Z|S) = \nu^*$. Then, the corresponding encoder $p_{\bphi}(z|x)$ can achieve $I(X;Z|S) = \nu^*$.
Moreover, by Lemma \ref{lem: MIBound}, an $\epsilon$-LDP $P_{Z|X}$ mapping leads to $I(X;Z) \leq \epsilon$. Hence, we have
\begin{align}
\Omega &= I(X;Z) - I(X;Z|S)\\
&\leq \epsilon - \nu^*.
\end{align}

On the other hand, from the data processing inequality of Markov chain $U \to X \to Z$, the upper bound of utility $\Gamma$ can be obtained by
\begin{align}
    \Gamma = I(U;Z) \leq I(X;Z) \leq \epsilon,
\end{align}
We also know that the utility constraint of \eqref{eq: method} lower bounds the leakage $\Gamma$ by $\gamma$. Thus, the lower bound of $\Gamma$ is established by the utility constraint of \eqref{eq: method}, $\Gamma \geq \gamma$.
Combining these inequalities, we have $\epsilon \geq \Gamma \geq \gamma$. 

Therefore, the mapping $P_{Z|X}$ satisfies $\epsilon$-LDP and can achieve a utility-leakage pair $(\Gamma, \Omega)$ with $\epsilon \geq \Gamma \geq \gamma$ and $\Omega \leq \epsilon-\nu^*$.

This completes the proof.
\end{proof}

\begin{figure}[t]
		\centering
		\includegraphics[width=0.3\linewidth]{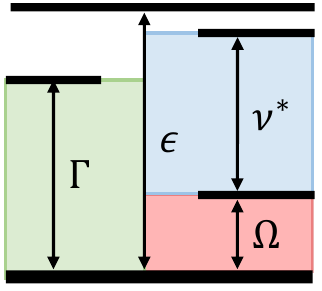}
		\caption{ Relation between $\epsilon$, $\Omega$, and $\nu^*$.}
		\label{fig:illu-PUT}
\end{figure}

Theorem~\ref{th: main} demonstrates that by employing an $\epsilon$-LDP randomizer $\cM$,  
we can obtain an $\epsilon$-LDP mapping $P_{Z|X}$ with utility $\Gamma$ and sensitive information leakage $\Omega$ both upper bounded by the privacy budget $\epsilon$.
Furthermore, the inequality $\Omega \leq \epsilon - \nu^*$ indicates that $\epsilon$ not only serves as an upper bound, but also enables the maximization of $g(\gamma, P_{USX}, \cM)$ to reduce leakage $\Omega$, as illustrated in Figure~\ref{fig:illu-PUT}. The use of a more stringent privacy-preserving mechanism $\cM$ can lead to an encoder with reduced sensitive information leakage. However, as the data processing inequality holds, the utility constraint $I(Z;U) \geq \gamma$ necessitates sufficiently large values of $\epsilon$ to ensure that the representation $Z$ contains adequate utility information. As such, the $\epsilon$-LDP randomizer $\cM$ is capable of effectively controlling the utility-leakage tradeoff, as validated in later experiment sections (see Figures~\ref{fig:AD} and~\ref{fig: MI}).

To further investigate the impact of $\epsilon$ on the utility-leakage tradeoff, the following corollaries formally demonstrate that specific values of $\epsilon$ can lead to optimal tradeoffs and perfect obfuscation.
\begin{corollary}
If $\epsilon = 0$, the utility-leakage pair $(\Gamma, \Omega)$ can only attain values of $\Gamma = 0$ and $\Omega = 0$.
\label{cor:zero}
\end{corollary}
\begin{proof}
    From data processing inequality of Markov chains $U \to X \to Z$ and $S \to X \to Z$, the upper bound of $I(U;Z)$ and $I(S;Z)$ can be obtained by 
    \begin{align}
        &0 \leq I(U;Z) \leq I(X; Z),\\
        &0 \leq I(S;Z) \leq I(X; Z).
    \end{align}
    And the Lemma \ref{lem: MIBound} states that $I(X; Z) \leq \epsilon$. If $\epsilon=0$, then $ I(S; Z) = I(X; Z) = 0$ and $I(U;Z) = I(X; Z) = 0$. Therefore, the obtained mapping only achieve $(\Gamma, \Omega)$ with $\Gamma = 0$ and $\Omega = 0$.
\end{proof}
This corollary shows that if the privacy parameter $\epsilon$ is set to zero, the only achievable utility-leakage pair is $(\Gamma, \Omega) = (0,0)$, which corresponds to the perfect obfuscation but no utility. This result is straightforward and establishes the equivalence between perfect LDP and perfect obfuscation.  
\begin{corollary}
If $\epsilon = \gamma$, the obtained encoder can achieve the optimal $(\Gamma, \Omega)$ characterized by \eqref{eq:optimal}.
  \label{cor:optimal}
\end{corollary}
\begin{proof}
    The optimal $(\Gamma, \Omega)$ is characterized by optimization \eqref{eq:optimal}, defined as, 
    \begin{align}
        G(\gamma, P_{USX}) = \inf\limits_{P_{Z|X}: I(U; Z) \geq \gamma } I(S; Z). 
    \end{align}
    As we presented before, the data processing inequality implies $I(U; Z) \leq I(X; Z)$. 
    If $\epsilon = \gamma$, then $I(X;Z) = \epsilon$ since $\gamma \leq I(U; Z)\leq I(X;Z)$ and $I(X; Z) \leq \epsilon$ from Lemma \ref{lem: MIBound}. 
    
    When $\gamma = \epsilon$, the optimization problem $g(\gamma, P_{USX}, \cM)$ can be formulated by:
    \begin{align}
        g(\gamma, P_{USX}, \cM) &= \max\limits_{\substack{p_{\bphi}(\hat{z}|x):\\Z = \cM(\hZ) \\ I(U; Z) \geq \gamma} } I(X; Z|S) \\
        &= \min\limits_{\substack{p_{\bphi}(\hat{z}|x):\\Z = \cM(\hZ) \\ I(U; Z) \geq \gamma} } \epsilon- I(X; Z|S)\\
        &= \min \limits_{\substack{p_{\bphi}(\hat{z}|x):\\Z = \cM(\hZ) \\ I(U; Z) \geq \gamma} } I(X;Z) -I(X; Z|S)\\
        &= \min\limits_{P_{Z|X}: I(U; Z) \geq \gamma} I(S; X) \\
        & = G(\gamma, P_{USX})
    \end{align}
    Thus, the optimization problem $g(\gamma, P_{USX}, \cM)$ is equal to the optimization problem $G(\gamma, P_{USX})$ when $\gamma = \epsilon$, producing an encoder that achieves the optimal $(\Gamma, \Omega)$.
\end{proof}

Corollary~\ref{cor:optimal} of Theorem~\ref{th: main} suggests that setting the privacy parameter $\epsilon$ equal to the desired utility constraint $\gamma$ leads to the attainment of the optimal trade-off between utility and sensitive information leakage, as characterized by $P_{Z|X}$ obtained through $g(\gamma, P_{USX}, \cM)$. Thus, a specific value of the privacy parameter for LDP can aid in achieving the best possible information obfuscation while satisfying the desired level of utility. This result can be beneficial in designing a fair representation learning system that requires LDP, as it allows the optimal utilization of randomness for privacy preservation to achieve optimal information obfuscation.
\paragraph{Discussion on privacy and fairness.}
The current empirical research examining the correlation between privacy and fairness has primarily concentrated on their conflicting impacts. Scholars have analyzed circumstances where fair learning models may compromise privacy~\cite{mozannar2020fair, jagielski2019differentially, tran2021differentially}, and conversely, where privacy-preserving techniques can hinder algorithmic fairness~\cite{fioretto2022differential, bagdasaryan2019differential, pujol2020fair, farrand2020neither}. However, our theoretical results contradict this conventional viewpoint by establishing a mutually beneficial relationship between privacy and fairness. Specifically, we demonstrate that introducing LDP mechanisms into the pre-processing stage can not only safeguard user-level privacy but also restrict sensitive information leakage during disclosure to ensure fairness. This is mainly because the scenario we considered (i.e., pre-processing) differs from prior empirical studies that directly incorporate private mechanisms into learning models (i.e., in-processing).
Nevertheless, our result still has significant implications for practical applications, where the privacy-preserving noise introduced in the original algorithm can be utilized to improve algorithmic fairness, contributing to the growing recognition that addressing privacy and fairness in ML requires a holistic approach considering their interplay.

\section{Non-Adversarial Variational Encoding}
\label{sec: variational}
Next, we present the method for developing the encoding process for information obfuscation with LDP, by solving the optimization Problem \eqref{eq: method}. 

Given a data source $P_{USX}$ and an $\epsilon$-LDP randomizer $\cM$, we utilize the Lagrangian of this problem and introduce a non-negative Lagrange multiplier $\beta \geq 0 \in \bR$. This optimization problem is then equivalent to maximizing the following objective function $\cL(\phi)$:
\begin{align}
    \cL (\phi)&=  I(X; Z|S) + \beta I(U; Z) \\
    & = H(X|S) - H(X|Z;S) + \beta [H(U)-H(U|Z)]\\
    &\equiv \bE_{p(s, u, x)} \{ \bE_{p_{\bphi}(z|x)}[ \log p_{\bphi}(x|z, s) ]\\
    & \qquad \qquad \qquad + \beta \bE_{p_{\bphi}(z|x)} [ \log p_{\bphi}(u|z)]\},\label{eq: loss}
\end{align}
where the last equivalence is obtained by omitting the constant terms $H(X|S)$ and $H(U)$. In \eqref{eq: loss}, the first term is the reconstruction loss given $z$ and $s$, and also quantifies the amount of uncertainty for potential adversarial inference due to $H(S|Z) \geq H(S) -\epsilon + I(X; Z|S)$, as presented in Theorem~\ref{th: main}. The second term represents the accuracy of the inference model with the obfuscated representation $z$. Hence, the parameter $\beta$ controls the tradeoff between utility and information obfuscation.

However, the posterior distributions $p_{\bphi}(x|z,s)$ and $p_{\bphi}(u|z)$ are computationally intractable. To address this issue, we resort to the variational Bayesian method to approximate the posterior distributions. Specifically, two variational distributions $q_{\btheta}(x|z, s)$ and $q_{\btheta}(u|z)$ are introduced to approximate $p_{\bphi}(x|z,s)$ and $p_{\bphi}(u|z)$, respectively, which are parameterized by $\btheta$.  
Therefore, we obtain the variational lower bound $\cL_{V}(\bphi, \btheta)$ of $\cL(\bphi)$ as follows (For a detailed derivation of the variational lower bound, please refer to Appendix~\ref{app: variational}).
\begin{align} 
    \cL(\bphi) &\geq \cL_{V}(\bphi, \btheta) \notag \\
    & \triangleq \bE_{p(s, u, x)} \{ \bE_{p_{\bphi}(z|x)}[ \log q_{\btheta}(x|z, s) ] \\
    &\qquad \qquad \qquad \quad+ \beta \bE_{p_{\bphi}(x|z)} [ \log q_{\btheta}(u|z)]\}.
    \label{eq: variational}
\end{align}
Moreover, we denote $q_{\theta}(x|z, s)$ and $q_{\theta}(u|z)$ as \emph{utility decoder} and \emph{side decoder}, respectively, and depict the proposed variational encoding framework in Figure~\ref{fig:model}.

\begin{figure}[t]
    \centering
    \includegraphics[width=8.cm]{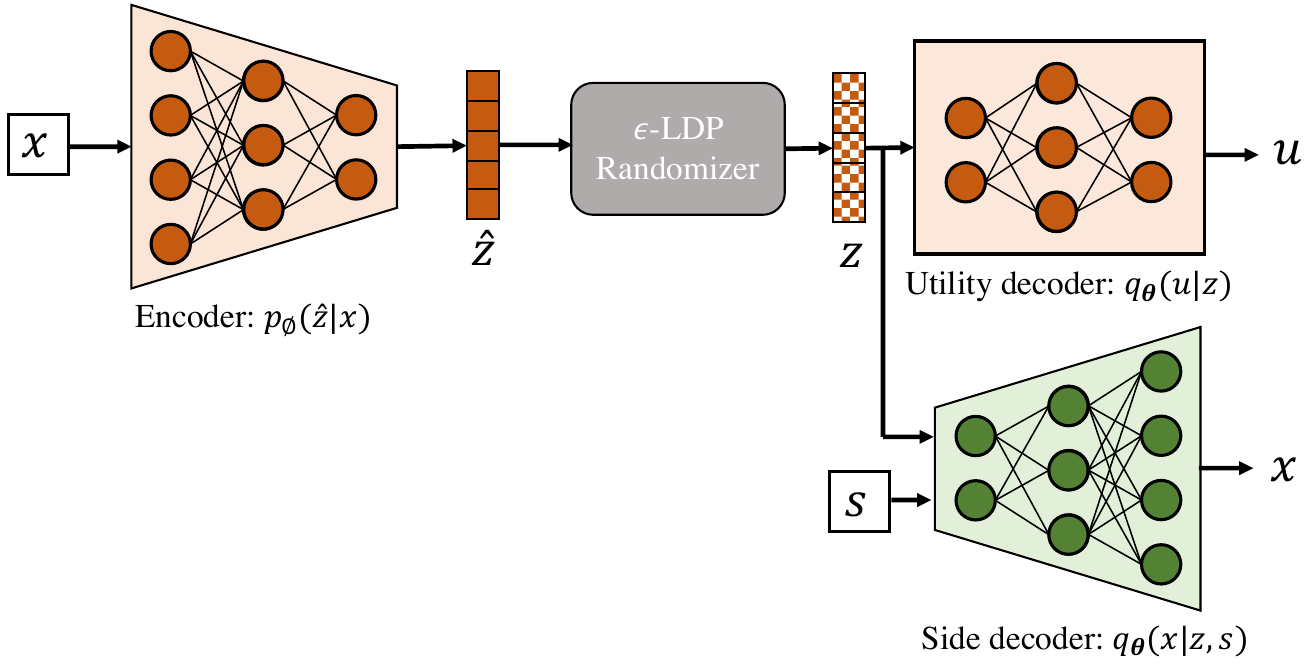}
    \caption{The proposed variational encoding framework.}
    \label{fig:model}
\end{figure}
\begin{remark}[Non-adversarial learning]
Our proposed framework for encoding fair representations departs from adversarial techniques, unlike the prevalent fair representation learning methods. While adversarial techniques are widely adopted, they cannot provide a general upper bound for sensitive information leakage and suffer from stability issues in training. 
Our approach, on the other hand, is non-adversarial and directly restricts the leakage by an upper bound of $I(S;Z) \leq \epsilon - I(X;Z|S)$, leading to a learned encoder that is more impervious to any form of adversary. 
\end{remark}
\begin{remark}[Without variational prior]
    Variational approaches for bottleneck problems often struggle to propose an appropriate prior that conforms well with the learned aggregated posterior $p_{\bphi}(z) = \bE_{p(x)}[p_{\bphi}(z|x)]$. However, the proposed variational encoding framework performs variational optimization without imposing any prior on the intractable distribution $p_{\bphi}(z)$. Therefore, the optimization does not restrict the marginal distribution $p_{\bphi}(z)$ to be similar to any pre-defined prior, allowing it to generate high-quality, well-obfuscated representations $z$ across the randomizer.
\end{remark}

Given a dataset of $N$ data points $\{ x^{(i)}\}_{i=1}^N$, as well as the corresponding samples of utility and sensitive variables $\{u^{(i)}, s^{(i)} \}_{i=1}^N$, we can now form the Monte Carlo estimation for $\cL_V(\bphi, \btheta)$ by sampling $L$ realizations of intermediate representation $\hat{z}$ from encoder $p_{\bphi}(\hat{z}|x)$ with the $\epsilon$-LDP mechanism $\cM$. We have the Monte Carlo estimate $\tilde{\mathcal{L}}_V(\bphi, \btheta) \simeq \cL_V(\bphi, \btheta)$:
\begin{align}\label{eq: MC}
    \tilde{\cL}_V(\bphi, \btheta) = \frac{1}{N} \sum\limits_{i=1}^N\{\frac{1}{L}\sum\limits_{l=1}^L[& \log q_{\btheta}(x^{(i)}|\cM(\hat{z}^{(i, l)}), s^{(i)})\\
    &+ \beta \log q_{\btheta}(u^{(i)}|\cM(\hat{z}^{(i,l)}))] \}, \notag\\
    \text{where  } \hat{z}^{(i,l)} \sim p_{\bphi}(\hat{z}|x^{(i)}).
\end{align}
Note that the obtained representation $z$ can be either continuous or discrete\footnote{We utilize the discretization part of VQ-VAE in discrete representation encoding and incorporate the quantization term into our loss \eqref{eq: MC} during the training phase. See Appendix~\ref{app: VQVAE} for more details.}. Then, we will present two typical $\epsilon$-LDP mechanisms, Laplace mechanism~\cite{dwork2014algorithmic} and randomized response~\cite{warner1965randomized}, for continuous and discrete representation learning, respectively. The detailed proof of the proposed $\epsilon$-LDP mechanisms can be found in Appendix~\ref{app: LM} and~\ref{app: RR}.
\paragraph{Continuous representation with a Laplace mechanism.} Considering the $d$-dimensional continuous representation $\hat{z} = (\hat{z}_1, \hat{z}_2, \dots, \hat{z}_d) \in \bR^d$, a local truncation is adopted to bound the sensitivity of $\hat{z}$ to achieve a better performance. Specifically, a threshold $t$ is chosen for each dimension $\hat{z}_i$ such that $-t \geq x_i \geq t$ for each $i\in [d]$. For such a representation $\hat{z}$, a Laplace mechanism defined in~\cite{dwork2014algorithmic} can be used to guarantee $\epsilon$-LDP by adding artificial Laplace noises, 
\begin{align}
    z = \hat{z} + (n_1, n_2, \dots, n_d),
    \label{eq: lap}
\end{align}
for each $\hat{z}$, $n_i \sim \text{Laplace}(0, 2td/\epsilon)$.
\paragraph{Discrete representation with randomized response mechanism.} When $\hat{z}$ is a $d$-dimensional discrete representation $\hat{z} = (\hat{z}_1,\hat{z}_2,\dots, \hat{z}_d) \in [k]^d$, the basic LDP mechanism, randomized response, can be applied on $\hat{z}$ to produce the stochastic representation $z$. To achieve $\epsilon$-LDP, the randomized response works as follows, for any $i \in [d]$ and $u \in [k]$,
\begin{align}\label{eq: RR}
    \Pr[z_i = v|u] = \left\{  
\begin{array}{lr}
    \frac{e^{\epsilon/d}}{e^{\epsilon/d}+k-1}, & \text{if } v = u,    \\
    \frac{1}{e^{\epsilon/d}+k-1} , &\text{if } v \not = u. 
\end{array}
\right. 
\end{align} 
It should be noted that the above mechanism \eqref{eq: RR} is consistent with the traditional randomized response~\cite{warner1965randomized} when $k=2$ and $d=1$.

\section{Related Work}
\paragraph{Differential privacy in ML} DP was originally proposed by~\cite{dwork2006calibrating} to protect the privacy of collected data and has been extensively used in data mining~\cite{friedman2010data} and federated learning~\cite{wei2020federated}. LDP~\cite{evfimievski2003limiting,cormode2018privacy} is a stronger notion of DP without assuming a trust data collector, which has been widely adopted in many practical applications such as web browsing behavior collection in Google's RAPPOR system~\cite{erlingsson2014rappor} and telemetry data collection in Microsoft~\cite{ding2017collecting}. While DP only establishes formal guarantees for defending against membership query attacks, it still suffers from attribute inference attacks, e.g., the sensitive attributes of data owners can be accurately inferred from differential-private datasets by some adversaries. 

\paragraph{Fair representation learning} 
There is a rich body of literature studying fairness in ML, which promotes fairness among individuals, subgroups, or groups, and can be categorized into three categories: Pre-processing, in-processing, and post-processing~\cite{mehrabi2021survey}.
Many in-processing and post-processing methods are designed for specific models and have been proposed by regularizing the models during the training stage or by modifying the trained models, respectively.
While fair representation learning is the pre-processing method that aims to ensure strong algorithmic fairness guarantees for any downstream inference models. 
The concept of fair representation learning was originally proposed by~\cite{zemel2013learning} to encode data for utility maximization while obfuscating any information about membership in the protected group. Thereafter, a plethora of different methods have been proposed and the majority of them are based on adversarial learning to obtain representations that are invariant across sensitive groups~\cite{zhao2019conditional,madras2018learning, bertran2019adversarially}. 
But the non-convexity of minmax problems in adversarial learning methods has always been a unique challenge for the convergence and stability of optimization. Besides, there also exist variational approaches such as the Variational Fair Autoencoder~\cite{louizos2015variational} and Conditional Fairness Bottleneck~\cite{rodriguez2021variational} where the VAE~\cite{kingma2013auto} was modified to produce fair encoding, as well as the methods with disentanglement~\cite{creager2019flexibly, sarhan2020fairness}. Nevertheless, it is important to note that a large majority of existing studies on fair representation learning have overlooked the aspect of differential privacy about raw data. In contrast, our proposed representation learning framework and variational encoding method are capable of achieving fairness while upholding the principles of LDP.

\paragraph{Bottleneck problems} The bottleneck problems are an important group of optimization problems with wide adoptions in machine learning and information theory. Originally, the first bottleneck problem, called IB, is introduced by~\cite{tishby2000information} to compress the data source while preserving the useful information about the target, which is widely used in compression~\cite{dai2018compressing}, representation learning~\cite{amjad2019learning}, and wireless communication~\cite{xie2022robust}. 
Due to the difficulty of mutual information optimization, the variants of IB such as variational IB \cite{alemi2016deep} and deterministic IB \cite{strouse2017deterministic} are proposed to develop tractable encoding frameworks. 
Following a similar principle of IB, the privacy funnel (PF)~\cite{du2012privacy} was proposed to minimize sensitive information leakage while maintaining the informativeness of representation. Inspired by PF, recent studies resort to conditional entropy of utility variables and sensitive variables to achieve information obfuscation, including conditional entropy bottleneck~\cite{fischer2020conditional} and conditional PF~\cite{rodriguez2021variational}. 

\paragraph{The interplay of fairness and privacy}
The intersection of DP and fairness is a topic of increasing interest. Jagielski et al.~\cite{jagielski2019differentially} initiated the study of fair learning under the constraint of differential privacy, specifically focusing on differential privacy and equalized odds. In a related study, Chang et al.~\cite{chang2021privacy} analyzed the privacy risks associated with group fairness using membership inference attacks, demonstrating the trade-off between fairness and privacy. Additionally, Tran et al.~\cite{tran2021differentially} proposed a Lagrangian dual approach to learn non-discriminatory predictors while preserving the privacy of individuals' sensitive information.
Another line of research endeavors to empirically showcase the potential negative impact of private mechanisms on fairness~\cite{fioretto2022differential, bagdasaryan2019differential, pujol2020fair, farrand2020neither}. However, there exists a dearth of methodological investigations into the harmonization of these two domains and the effective integration of differential privacy and fairness. Motivated by this gap, our work aims to explore the integration of privacy techniques into algorithmic fairness by focusing on fair representation learning with LDP.

\section{Experiments}
\label{sec: exp}
In this section, we justify our results for information obfuscation with LDP and evaluate the proposed variational method on several common fairness datasets\footnote{The experimental details including datasets, metrics, and network architectures can be found in the supplementary material.}.
Our experimental methodology commences with a toy example, which is aimed at explicating the tradeoff between utility and leakage in information obfuscation with LDP on the \texttt{colored-MNIST} dataset~\cite{lecun1998gradient}. 
We will then validate the tradeoffs and our theoretical results by numerical evaluation of real-world benchmark datasets. 
Lastly, we compare our proposed method with the representative fair representation learning baselines. 

\subsection{Experimental Setup}
\label{subsec:exp_set}

\begin{table}[t]
\centering
\caption{Basic statistics of the pre-processed datasets}
\begin{tabular}{cccc}
				\toprule
                       Dataset  & Train / Test & \# Attributes & $P(U=1)$   \\
				\midrule
    \texttt{adult}& $32561/16281$& 13  & $0.2362$  \\
    \texttt{compas} & $4320/1852$& 10  & $0.5378$ \\
    \texttt{hsls} & $10156/4353$& 57  & $0.5114$ \\
				\bottomrule
		\end{tabular}
  \label{tab:datasets}
\end{table}
\paragraph{Datasets.} We conduct experiments on one synthetic dataset \texttt{Colored-MNIST}, and three real-world benchmark datasets are widely used in fair ML problems, including the UCI \texttt{adult} dataset~\cite{Dua:2019}, the ProPublica \texttt{compas} dataset~\cite{dieterich2016compas}, and the \texttt{hsls} dataset~\cite{jeong2022fairness}. The basic information of the datasets is summarized as follows.
\begin{figure*}[t]
    \centering
    \includegraphics[width=0.7\textwidth]{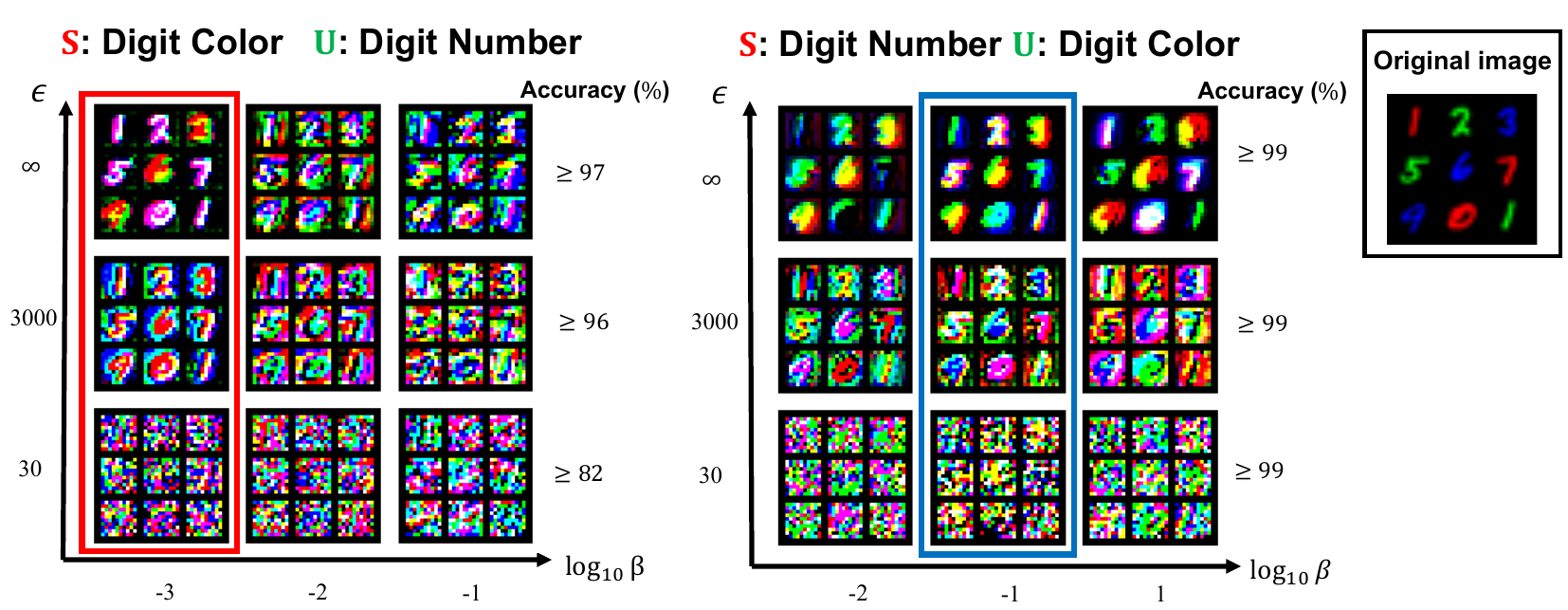}
    \caption{The encoded representations with $\epsilon$-LDP Laplace mechanism for random samples in \texttt{Colored-MNIST} dataset. {As $\epsilon$ decreases, the representations on the left become increasingly noisy, leading to more obfuscation of the color attribute (e.g., samples enclosed by a red box). Meanwhile, the representations on the right also exhibit increased noise, yet the color correlation persists (e.g., samples enclosed by a blue box).}}
    \label{fig: toy}
\end{figure*}
\begin{itemize}
    \item \texttt{Colored-MNIST}: MNIST dataset~\cite{lecun1998gradient} is a collection of handwritten digits containing 70,000 grayscale images of size 3 × 28 × 28, which is commonly used for machine learning systems. The \texttt{Colored-MNIST} is the colored version of MNIST dataset, where digits are randomly colored into red, blue, or green. 
    \item \texttt{adult}: The UCI \texttt{Adult} dataset~\cite{Dua:2019} contains 48,842 instances based on census data, which is described by 14 attributes and a binary target variable \textit{income} indicating whether a person's annual income is larger than 50K dollars. Our experiment setup follows ~\cite{zemel2013learning}, with \textit{gender} as the sensitive variable $S$ and \textit{income} as the utility variable $U$. 
    \item \texttt{compas}:  The \texttt{compas} (Correctional Offender Management Profiling for Alternative Sanctions) dataset~\cite{dieterich2016compas} contains 6172 convicted criminal instances with 12 attributes. 
    For this dataset, we use the processed binary attribute \textit{race} (whether the instance is African American) as the sensitive variable $S$ and recidivism as the utility variable $U$.
    \item \texttt{hsls}: The \texttt{hsls}~\cite{jeong2022fairness} dataset contains information about over 23,000 students from high schools in the USA. The features consist in information about the students’ demographic and academic performance, as well as data about the schools. The task is to predict exam scores while being fair with respect to race.
    \end{itemize}
    In the interest of a fair comparison, we utilize the same pre-processing techniques as in [4] for the \texttt{adult} and \texttt{compas} datasets and in Alghamdi et al. [2] for the \texttt{hsls} dataset. The basic statistics of processed real-world datasets are also summerized in Table~\ref{tab:datasets},

    \paragraph{Metrics.} Throughout the experiments, we use mutual information $I(S;Z)$ and $I(U;Z)$ to measure the sensitive information leakage and the utility derived from the learned representations. Furthermore, we employ the top-$1$ classification accuracy of the inference of side decoder to measure the utility provided by the learned representations.  In order to ascertain the efficacy of the proposed fair representation learning methodology, we also incorporate conventional metrics for assessing group fairness, namely the Demographic Parity Gap ($\Delta_{\text{DP}}$) and the Equalized Odds Gap ($\Delta_{\text{EO}}$), to evaluate the fairness of the inferences based on the learned representation.
    \paragraph{Compared Methods.} In order to assess the effectiveness of our proposed approach, we conduct a comprehensive comparison with prominent fair representation learning methodologies. These include adversarial training-based methods such as CFAIR~\cite{zhao2019conditional} and LAFTR~\cite{madras2018learning}, privacy funnel optimization-based method PPVAE~\cite{nan2020variational}, disentanglement-focused approach FFVAE~\cite{creager2019flexibly}, variational approaches CFB~\cite{rodriguez2021variational} and VFAE~\cite{louizos2015variational}, as well as the standard baseline raw data.
    \begin{figure*}[t]
		\centering
			\begin{subfigure}{0.33\textwidth}
				\includegraphics[width=.94\textwidth]{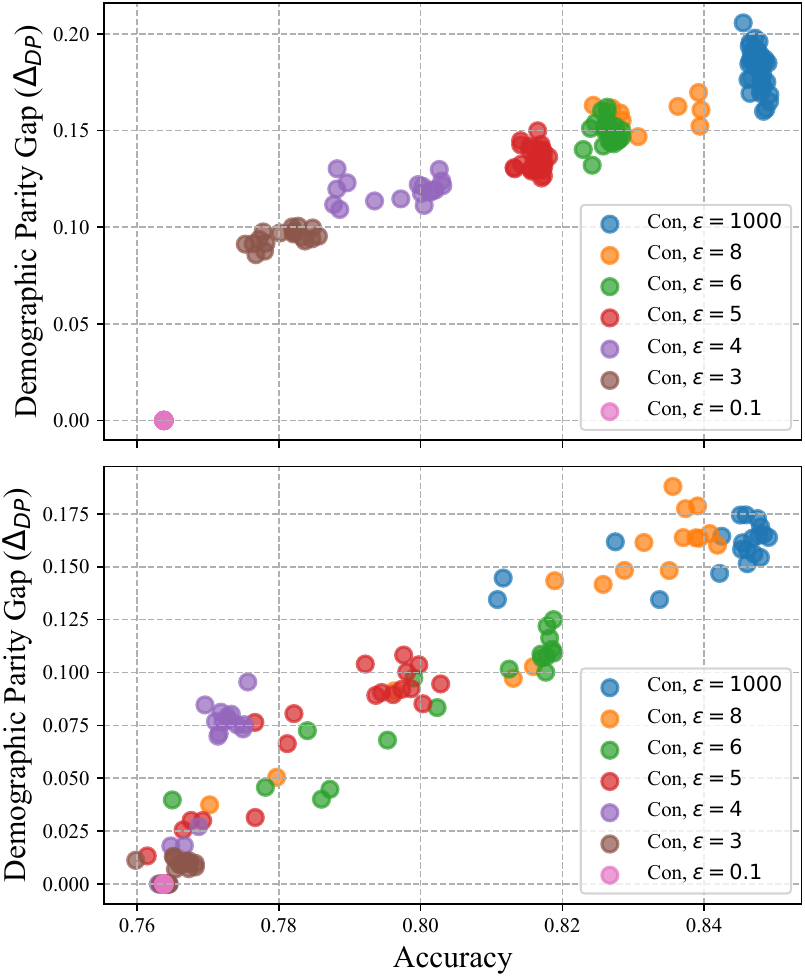}
				\caption{\texttt{adult}}
			\end{subfigure}
			\begin{subfigure}{0.33\textwidth}
				\includegraphics[width=.94\textwidth]{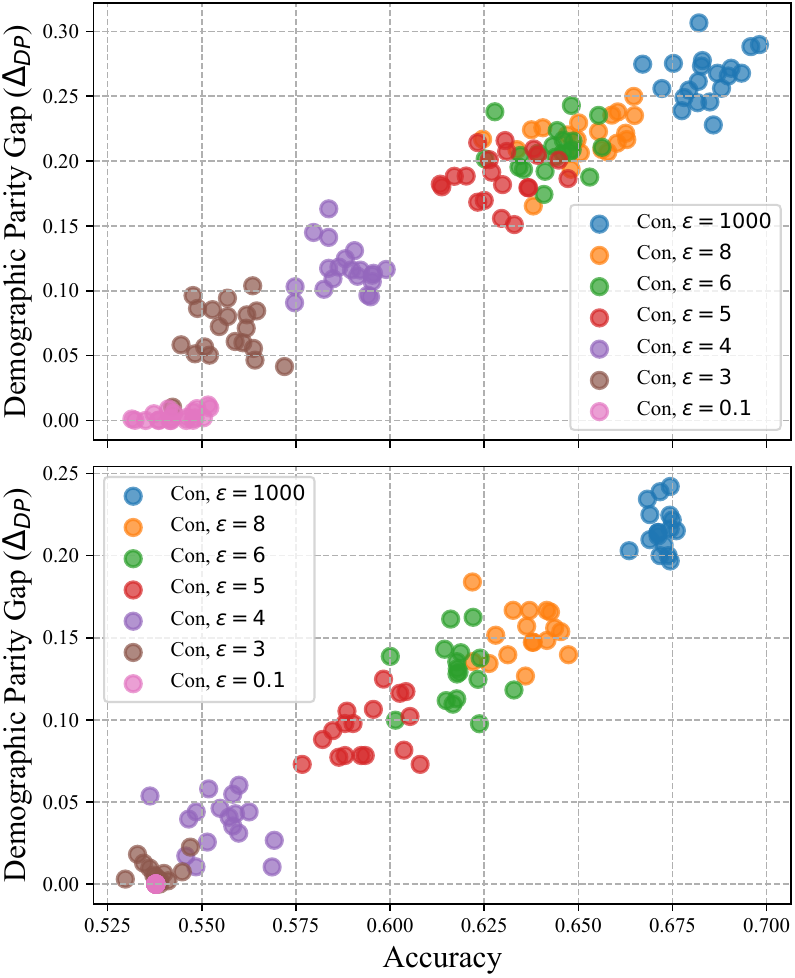}
				\caption{\texttt{compas}}
			\end{subfigure}
   \begin{subfigure}{0.33\textwidth}
		\includegraphics[width=.94\textwidth]{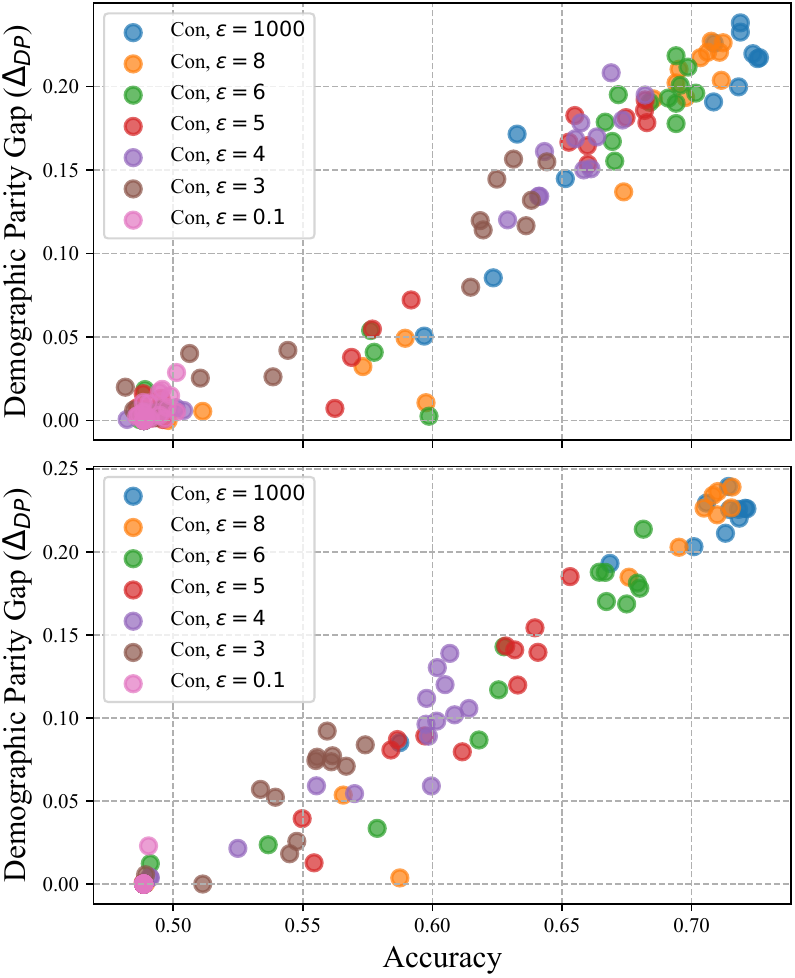}
				\caption{\texttt{hsls}}
			\end{subfigure}
  \caption{Accuracy-$\Delta_{\text{DP}}$ tradeoffs with varying values of $\epsilon \in [0, 10^3]$ and $\beta \in [0.1, 10^3]$ on (a) \texttt{adult}, (b) \texttt{compas}, and (c) \texttt{hsls} datasets.}
		\label{fig:AD}
\end{figure*}
    \paragraph{Implementation.} The procedures in the experiments are described as follows. Firstly, all the evaluated methods are trained on the training set until convergence, where the dimensions of representation $Z$ are uniformly set to $d=2$. Then, we adopt 1-hidden-layer fully-connected neural networks as the downstream inference models and train them on the representation obtained from each evaluated method to predict the utility $U$. For the proposed method, we take the utility decoder as the downstream inference model as it has the same architecture and is designed for the downstream task. Furthermore, we also train the downstream inference models on the representations obtained by the baselines with Laplace mechanisms to investigate the impact of LDP mechanisms on the baseline methods. Finally, the trained inference models are used to perform the prediction of utility $U$ on the test set. The task accuracy and fairness metrics are computed based on the prediction, and sensitive information leakage $I(U;Z)$ is estimated.
\subsection{Synthetic Dataset}
We first conduct the proposed method on the synthetic dataset, colored-MNIST, as a toy example to directly illustrate the utility-leakage and the role of LDP randomers in the encoding processes. In this toy example, we aim to construct representations $Z$ with the same dimensions as $X$. Subsequently, we train the proposed variational encoding approach using \eqref{eq: MC} with a $\epsilon$-LDP Laplace mechanism, as presented in \eqref{eq: lap}. We consider two scenarios for the data utility $U$ and the sensitive information $S$: \romannumeral1) $S$ represents color, and $U$ represents the digit number, \romannumeral2) $S$ represents digital number, and $U$ represents color. Note that a high dimensionality would result in more Laplacian noise injection, leading to excessively obfuscated representations that might hinder the visual understanding of the trade-off trends. For ease of observation in this toy example, we made some design choices including downsampling images and using relatively large values for $\epsilon$.

Figure \ref{fig: toy} illustrates the obtained representations for various $\beta$ and $\epsilon$ values, along with the corresponding inference accuracies of the side decoders for randomly selected samples of $X$. As the value of $\epsilon$ increases, the representations become increasingly obscured, making it more challenging to discern distinctions between them. This observation aligns with the inherent nature of LDP randomization.
Moreover, lower values of $\epsilon$ contribute to more effective obfuscation of sensitive information, thereby validating our theoretical results. For instance, as $\epsilon$ decreases, the representations on the left figure exhibit increased noise, resulting in greater obfuscation of the color attribute (e.g., samples enclosed by a red box). Conversely, the representations on the right also display heightened noise, yet the color correlation persists (e.g., samples enclosed by a blue box).
Nevertheless, the inference accuracy remains consistently high for $U$, regardless of the chosen $\epsilon$ and $\beta$ values. This is primarily because the color attribute is independent of the digit number in the \texttt{Colored-MNIST} dataset. Consequently, both scenarios achieve satisfactory inference accuracy, thereby providing ample utility.
\begin{figure}[t]
\centering
\begin{subfigure}{0.235\textwidth}
	\includegraphics[width=0.99\textwidth]{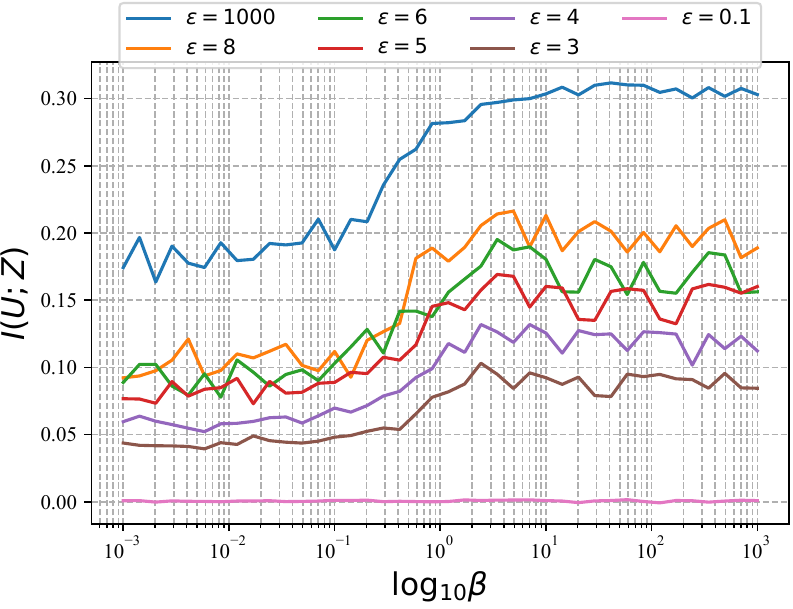}
				\caption{$I(U;Z)$}
\end{subfigure}
			\begin{subfigure}{0.235\textwidth}
			\includegraphics[width=0.99\textwidth]{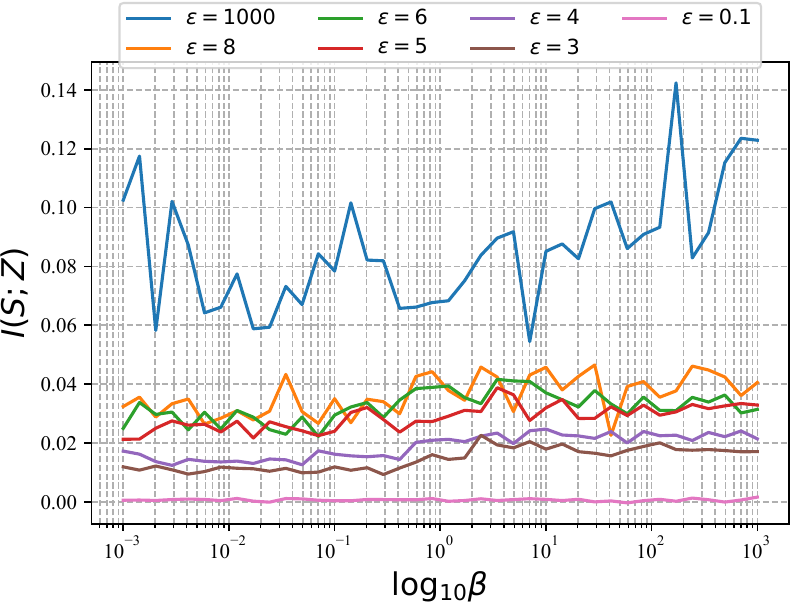}
				\caption{$I(S;Z)$}
			\end{subfigure}
   
   \vspace{5pt}
   \begin{subfigure}{0.235\textwidth}
			\includegraphics[width=1.\textwidth]{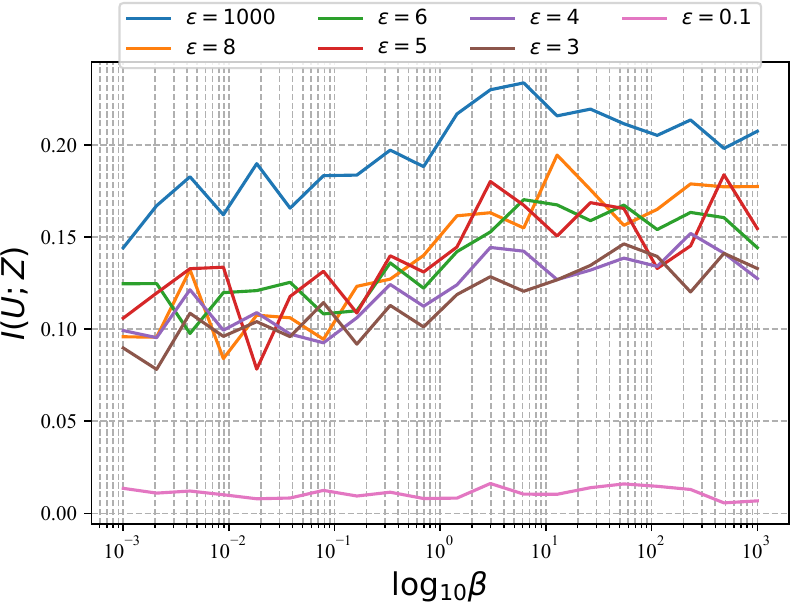}
				\caption{$I(U;Z)$}
			\end{subfigure}
   \begin{subfigure}{0.235\textwidth}
			\includegraphics[width=1.\textwidth]{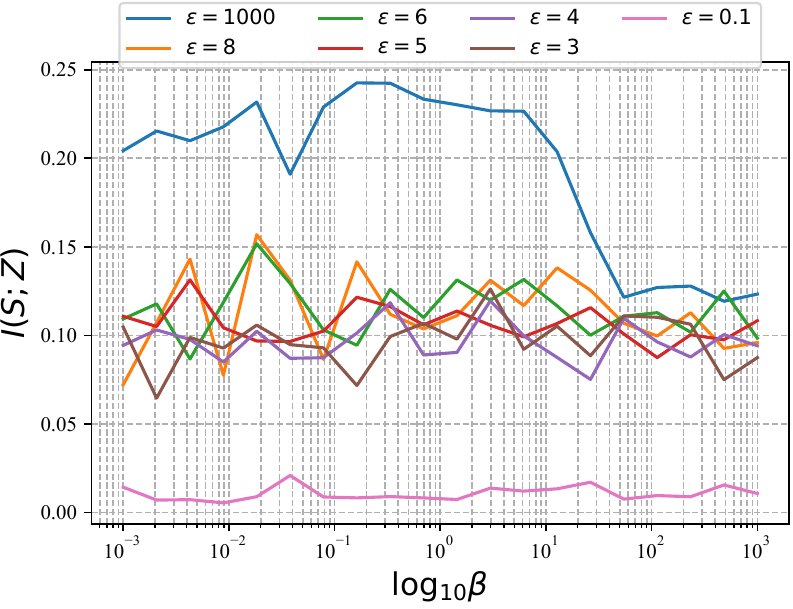}
				\caption{$I(S;Z)$}
			\end{subfigure}
   
   \vspace{5pt}
   \begin{subfigure}{0.235\textwidth}
			\includegraphics[width=1.\textwidth]{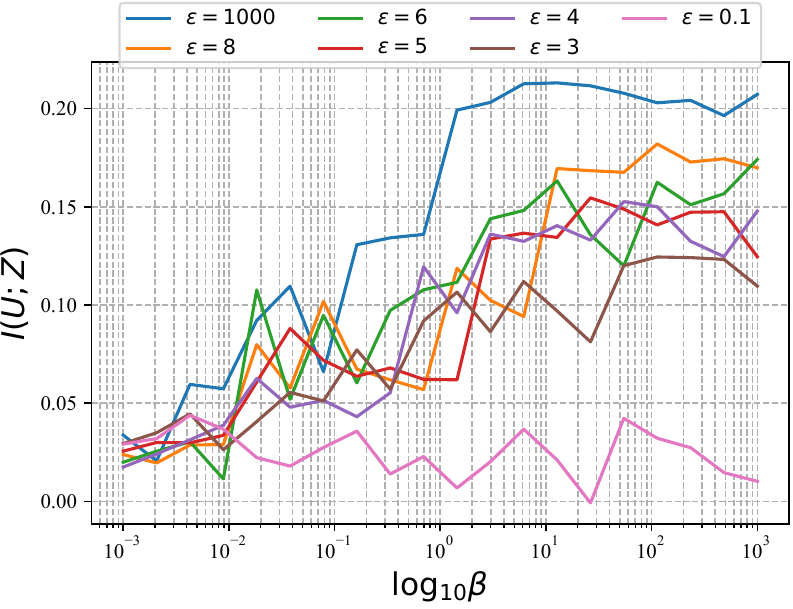}
				\caption{$I(U;Z)$}
			\end{subfigure}
   \begin{subfigure}{0.235\textwidth}
			\includegraphics[width=1.\textwidth]{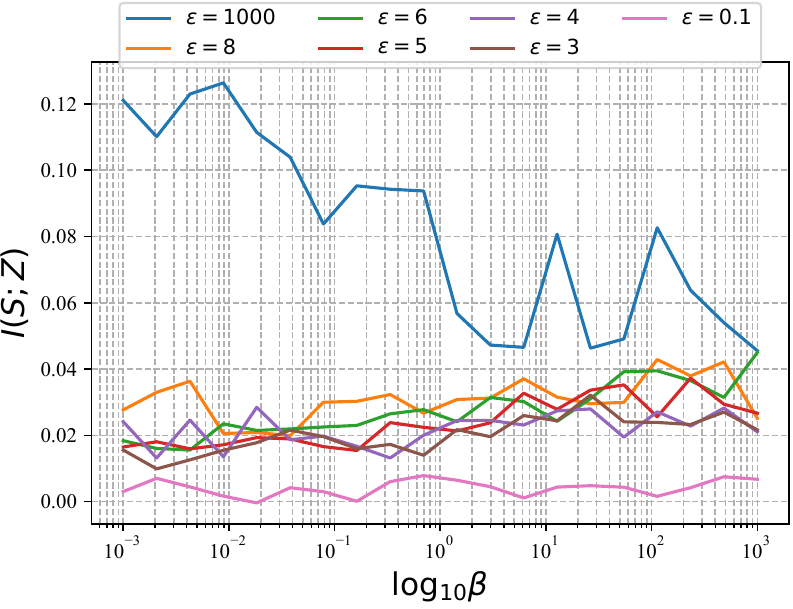}
				\caption{$I(S;Z)$}
			\end{subfigure}
\caption{ Utility $I(U;Z)$ and sensitive information leakage $I(S;Z)$ of continuous $Z$ with $\beta \in [10^{-3}, 10^3]$ on the test data of (top) \texttt{adult}, (middle) \texttt{compas}, and (bottom) \texttt{hsls} dataset.}
\label{fig: MI}
\end{figure}
\subsection{Real-World Datasets}
In this section, we validate our results on three real-world datasets, \texttt{adult}, \texttt{compas}, and \texttt{hsls}.
We learn the continuous and discrete 2-dimensional obfuscated representation $Z$ using optimization~\eqref{eq: MC} with the $\epsilon$-LDP Laplace mechanism and randomized response, which we denote as \emph{Con} and \emph{Dis}, respectively. 
We adopt a 1-hidden-layer neural network as the utility decoder and use its accuracy to measure utility.

\subsubsection{Utility-Leakage Tradeoff with LDP-Mechanisms}
Figure~\ref{fig:AD} depicts the tradeoff graphs representing the relationship between inference accuracy and the level of fairness ($\Delta_{\text{DP}}$) across a range of $\epsilon$ values spanning the interval $[0, 10^3]$, as well as varying $\beta$ values within the range $[0.1, 10^3]$. Each data point is obtained by systematically adjusting the values of $\beta$.

Firstly, it is evident that as the inference accuracy of the side decoder increases, the values of $\Delta_{\text{DP}}$ also rise, indicating the inherent tradeoff between utility and fairness in the context of fair representation learning. While the proposed methods allow for the attainment of diverse utility-fairness tradeoffs by adjusting the values of $\beta$, it is notable that the obfuscated representations $Z$ with smaller values of $\epsilon$ can readily achieve lower values of $\Delta_{\text{DP}}$ at the expense of accuracy. This observation confirms the existence of a utility-leakage tradeoff and underscores the limitations imposed on sensitive information leakage by LDP randomizers, as articulated in Theorem~\ref{th: main}.
Notably, when $\epsilon$ approaches 0, the accuracy becomes equivalent to the group ratio, resulting in negligible utility. 
Furthermore, it is worth noting that the discrete representation $Z$ is more effectively obfuscated by the randomized response mechanism compared to the continuous representation $Z$ with Laplace mechanisms. This disparity arises due to the inherent limitations on information capacity imposed by the cardinality of discrete representations, rendering them amenable to effective obfuscation. Additionally, the exhibited tradeoffs in Figure~\ref{fig:AD} may vary depending on the specific dataset, as the dependence between the sensitive attributes and the targets can differ. Nonetheless, the overall trend remains consistent.

To further investigate the utility-fairness tradeoffs influenced by $\beta$ and $\epsilon$, we present the utility measure $I(U;Z)$ and sensitive information leakage $I(S;Z)$ learned for varying $\beta$ values within the range $[10^{-3}, 10^3]$. Figure~\ref{fig: MI} shows the obtained results. 
The results demonstrate that increasing the value of $\beta$ leads to an increase in the utility $I(U;Z)$. However, the sensitive information leakage $I(S;Z)$ shows only minor changes when the hyperparameter $\beta$ varies. 
This observation suggests that the hyperparameter $\beta$ more effectively governs the utility of $Z$ compared to the sensitive information leakage. 
Furthermore, the introduction of $\epsilon$-LDP randomizers proves to be more effective in constraining sensitive information leakage. From a practical perspective, one can set up an LDP randomizer with a small value of $\epsilon$ to establish an upper constraint on $I(S; Z)$, and increase the value of $\beta$ to utilize IB-based optimization to suppress the sensitive information leakage.
These experimental results validate the necessity of incorporating LDP randomizers in the encoding process and provide empirical support for our theoretical findings that the LDP randomizers in fact facilitate the fairness of the learned representations on the real-world datasets.

\begin{table*}[t]
		\centering
 		\caption{Performance of evaluated methods on \texttt{adult}, \texttt{compas}, and \texttt{hsls} datasets. The values for the baselines are depicted as: without/with $\epsilon$-LDP Laplace mechanisms, and the values of $\epsilon$ are the same for each dataset. The last column is the accuracy of a random forest classifier which reconstructs the sensitive attribute from the observed representation.The best and second-best results are bolded and underlined, respectively. $^*$: The accuracy is equal to the group ratio, which means the learned representations provide no utility. Each experiment is replicated it 10 times with identical hyperparameters setting. }
		\label{tab: compared-exp}
		\resizebox{1.00\linewidth}{!}{
			\begin{tabular}{ccccccc}
				\toprule
                       Dataset & Methods & Accuracy & $I(S;Z)$ & $\Delta_{\text{DP}}$ &  $\Delta_{\text{EO}}$ & Accuracy ($S$)\\
				\midrule
    \multirow{8}{*}{\texttt{compas}}& Ours (Con, $\epsilon=5$) &  $0.64 \pm{0.01}$& $\mathbf{0.05 \pm{0.01}}$  & $0.15\pm{0.02}$ & $0.14\pm{0.02}$ & {$\underline{0.52\pm{0.01}}$}\\
    & Ours (Dis, $\epsilon=5$) & $0.63\pm{0.01}$&$\underline{0.06\pm0.02}$&$\underline{0.11\pm0.01}$ &$ \mathbf{0.12 \pm0.02}$&$\mathbf{0.51\pm0.01}$\\
    & CFAIR & 0.67 {$\pm{0.01}$}/0.62 {$\pm{0.02}$} &0.10 {$\pm{0.03}$}/0.11 {$\pm{0.04}$} & 0.31 {$\pm{0.03}$}/0.19 {$\pm{0.03}$}&0.32 {$\pm{0.06}$}/0.20 {$\pm{0.01}$} & {0.67$\pm{0.00}$}/ {0.56$\pm{0.00}$}\\
    & PPVAE & 0.64 {$\pm{0.01}$}/0.56 {$\pm{0.01}$}&0.24 {$\pm{0.08}$}/0.10 {$\pm{0.02}$} &0.25 {$\pm{0.02}$}/0.02 {$\pm{0.01}$} &0.33 {$\pm{0.00}$}/0.02 {$\pm{0.01}$}& {0.90$\pm{0.01}$}/ {0.55$\pm{0.02}$}\\
    & FFVAE & 0.54$^*$/0.54$^*$& - / -&- / - &- / -&- / -\\
    & CFB &0.65 {$\pm{0.01}$}/0.57 {$\pm{0.01}$} &0.15 {$\pm{0.03}$}/0.04 {$\pm{0.00}$} &0.22 {$\pm{0.01}$}/0.04 {$\pm{0.01}$} &0.24 {$\pm{0.01}$}/0.04 {$\pm{0.00}$} &  {0.59$\pm{0.02}$}/ {$0.51\pm{0.01}$} \\
    & VFAE&0.64 {$\pm{0.01}$}/0.56 {$\pm{0.01}$} &0.24 {$\pm{0.10}$}/0.11 {$\pm{0.01}$} &$\mathbf{0.09\pm0.03}$/0.07 {$\pm{0.01}$} &$\underline{0.13 \pm0.01}$/0.09 {$\pm{0.01}$}& {0.91$\pm{0.02}$}/ {$0.51\pm{0.00}$}\\
    & Raw data & 0.67 {$\pm{0.01}$}/0.55 {$\pm{0.00}$}& 0.30 {$\pm{0.02}$}/0.27 {$\pm{0.01}$}& 0.23 {$\pm{0.00}$}/0.04 {$\pm{0.02}$}&0.23 {$\pm{0.00}$}/0.04 {$\pm{0.02}$}&  {$0.69\pm{0.00}$}/ {$0.63\pm{0.00}$}\\
    \midrule
    \multirow{8}{*}{\texttt{adult}}& Ours (Con, $\epsilon=5$) & 0.82 {$\pm{0.00}$} & $\underline{0.03\pm0.00}$ & 0.12 {$\pm{0.01}$} & 0.06 {$\pm{0.01}$} & $\mathbf{0.61\pm0.00}$\\
    & Ours (Dis, $\epsilon=5$) & 0.78 {$\pm{0.00}$}& $\mathbf{0.02\pm0.01}$& $\mathbf{0.05\pm0.01}$& $\underline{0.03\pm0.01}$&$\mathbf{ 0.61\pm0.00}$\\
    & CFAIR & 0.80 {$\pm{0.01}$}/0.79 {$\pm{0.01}$}& 0.06 {$\pm{0.02}$}/0.04 {$\pm{0.01}$}& 0.24 {$\pm{0.01}$}/0.14 {$\pm{0.01}$}&0.06 {$\pm{0.03}$}/0.04 {$\pm{0.01}$} &  $\underline{0.63\pm0.01}$/ {0.62$\pm{0.00}$} \\
    & PPVAE & 0.79 {$\pm{0.00}$}/0.76$^*$&0.18 {$\pm{0.02}$}/ - &0.08 {$\pm{0.01}$}/ - &0.12 {$\pm{0.00}$}/ - &  {$0.66\pm{0.01}$}/ -\\
    & FFVAE & 0.76$^*$/0.76$^*$ &- / - &- / - &- / -&- / -\\
    & CFB &0.82 {$\pm{0.01}$}/0.77 {$\pm{0.00}$} &0.06/0.01 {$\pm{0.00}$} &0.11 {$\pm{0.02}$}/0.02 {$\pm{0.01}$} &$\mathbf{0.02\pm0.01}$/0.01 {$\pm{0.00}$}&  {$0.64\pm{0.02}$}/ {$0.61\pm{0.00}$}\\
    & VFAE&0.78 {$\pm{0.01}$}/0.76$^*$ &0.16 {$\pm{0.02}$}/ -  &$\underline{0.06\pm0.01}$/ -  &0.06 {$\pm{0.01}$}/ -  &  {$0.71\pm{0.01}$}/ -\\
    & Raw data &0.85 {$\pm{0.00}$}/0.76$^*$& 0.58 {$\pm{0.01}$}/ - & 0.18 {$\pm{0.00}$}/ - & 0.07 {$\pm{0.00}$}/ - & {$0.84\pm{0.00}$}/ - \\
    \midrule
    \multirow{8}{*}{\texttt{hsls}}& Ours (Con, $\epsilon=5$) & 0.68 {$\pm{0.01}$} & $\mathbf{0.01\pm0.00}$ & 0.12 {$\pm{0.01}$} & 0.09 {$\pm{0.01}$} &$\mathbf{0.53\pm0.01}$ \\
    & Ours (Dis, $\epsilon=5$) &0.67 {$\pm{0.02}$}& $\mathbf{0.01\pm 0.00}$ & ${0.09\pm 0.00}$ & 0.08 {$\pm{0.02}$}& $\underline{0.52\pm0.02}$\\
    & CFAIR & 0.51*/0.51* & -/- & -/-&-/- &  -/- \\
    & PPVAE & 0.62 {$\pm{0.02}$}/0.56{$\pm{0.03}$}&${0.17\pm0.06}$/ 0.48{$\pm{0.02}$} &$\underline{0.07\pm 0.01}$/ 0.05{$\pm{0.03}$} &$\mathbf{0.06\pm0.03}$/ 0.09{$\pm{0.02}$} &  {$0.81\pm{0.03}$}/ 0.58{$\pm{0.07}$}\\
    & FFVAE & 0.51$^*$/0.51$^*$ &- / - &- / - &- / -&- / -\\
    & CFB &0.61{$\pm{0.01}$}/0.59 {$\pm{0.01}$} &$\underline{0.03 \pm0.01}$ /0.03 {$\pm{0.01}$} &0.09 {$\pm{0.02}$}/0.07 {$\pm{0.02}$} &$\underline{0.07\pm0.01}$/0.06 {$\pm{0.03}$}&  {$0.54\pm{0.00}$}/ {$0.53\pm{0.02}$}\\
    & VFAE&0.61 {$\pm{0.02}$}/0.52{$\pm{0.01}$}$^*$ &0.33 {$\pm{0.02}$}/ 0.07{$\pm{0.02}$}  &$\mathbf{0.05\pm0.02}$/ 0.02{$\pm{0.01}$}  &$\mathbf{0.06 \pm0.01}$/ 0.02{$\pm{0.01}$}  &  0.75{$\pm{0.021}$}/ 0.53{$\pm{0.01}$}\\
    & Raw data &0.73 {$\pm{0.00}$}/0.51{$\pm{0.02}$}& 0.97 {$\pm{0.01}$}/0.97{$\pm{0.01}$}  & 0.15 {$\pm{0.00}$}/0.03{$\pm{0.02}$}  & 0.13 {$\pm{0.00}$}/ 0.04{$\pm{0.02}$} & {$0.91\pm{0.00}$}/ 0.60{$\pm{0.02}$} \\
				\bottomrule
	           \end{tabular}
            }
	\end{table*}
\subsubsection{Performance Comparison of Fair Representation Learning Methods}
We proceed to compare our methods with several representative baselines for fair representation learning, as outlined in Section~\ref{subsec:exp_set}. For the utility decoder and downstream classifier, we continue to employ a 1-hidden-layer neural network. To ensure a fair comparison, we tune the hyperparameters of all evaluated methods to achieve similar levels of accuracy. In addition to accuracy, we assess fairness metrics such as $I(S;Z)$, $\Delta_{\text{DP}}$, and $\Delta_{\text{EO}}$. Furthermore, we train an additional classifier based on the learned representations specifically to discern the sensitive attributes, denoted as \emph{Accuracy ($S$)}, which serves as an additional fair metric.
Given that the baselines do not incorporate LDP mechanisms, we directly apply Laplace mechanisms with $\epsilon=8$ to the obtained representations to examine the impact of LDP on conventional fair representation learning methods. 

As seen in Table~\ref{tab: compared-exp}, our methods not only ensure LDP but also present comparable fairness performance relative to the baselines with similar accuracy. Additionally, we discovered that FFVAE cannot provide any utility regardless of the hyperparameter tuning, which is primarily due to the restriction of the obfuscated representation to a low-dimensional space, $d=2$. Furthermore, we can observe that directly applying Laplace mechanisms to the baseline methods results in a significant decline in accuracy and an unstable reduction in fairness metrics. However, our proposed method can achieve both LDP and fairness with a sufficient accuracy requirement. This underscores the importance of a holistic design in the combination of LDP and fair representation learning, as separate implementation cannot achieve the desired effect.

\section{Discussions}
\label{sec: dis}
This work is driven by the recognition of the relationship between privacy and fairness, and seeks to harness their synergy to develop more effective and reliable algorithms. However, it is important to acknowledge the limitations of this study considering the overall objective.
\begin{itemize}
    \item Firstly, our theoretical results and the proposed methods are only focused on LDP and pre-processing methods for fairness in a bottleneck framework. While there exist other categories of fairness methods, both in-processing and post-processing, the extension of our approach to these methods and other privacy mechanisms remains an open question. 
    \item  Another limitation regards the proposed variational approach in Section~\ref{sec: variational}. Even though our method is developed without the restriction of variational priors, it still requires tractable density for the variational approximation, which limits our application scenarios, especially in the case of estimating MI in high-dimensional settings with finite samples. One future direction is to directly develop gradient estimators for implicit models~\cite{wen2020mutual, li2017gradient}.
    \item The side decoder of the proposed method appears redundant for obfuscated representation encoding, as the utility decoder can serve as an inference model for downstream tasks. One possible solution is to eliminate the side decoder by adopting contrastive learning techniques such as infoNCE~\cite{oord2018representation}, leave-one-out and contrastive log-ratio upper bounds~\cite{poole2019variational, cheng2020club}.
\end{itemize}

Despite these limitations, this research adds significant value to the existing body of knowledge on the interplay between privacy and fairness. It highlights the criticality of adopting a comprehensive approach in the development of trustworthy ML systems.

\section{Conclusion}
\label{sec: con}
In this work, we investigate the integration of differential privacy into algorithmic fairness, specifically focusing on fair representation learning with LDP. We employ a bottleneck framework to minimize the leakage of sensitive information and incorporate LDP randomization into the encoding process. 
Our theoretical findings have been validated through simulation results, confirming the effectiveness of our method in achieving LDP and fairness without compromising utility. This validation provides empirical evidence supporting the practical application of our approach. By demonstrating the successful integration of LDP and fairness while preserving utility, our research contributes to the broader understanding of the intricate relationship between privacy and fairness in ML.

{\footnotesize \bibliographystyle{acm}
\bibliography{sample}}

\begin{thebibliography}{10}

\bibitem{abadi2016deep}
{\sc Abadi, M., Chu, A., Goodfellow, I., McMahan, H.~B., Mironov, I., Talwar, K., and Zhang, L.}
\newblock Deep learning with differential privacy.
\newblock In {\em Proceedings of the 2016 ACM SIGSAC conference on computer and communications security\/} (2016), pp.~308--318.

\bibitem{alemi2016deep}
{\sc Alemi, A.~A., Fischer, I., Dillon, J.~V., and Murphy, K.}
\newblock Deep variational information bottleneck.
\newblock {\em arXiv preprint arXiv:1612.00410\/} (2016).

\bibitem{alghamdi2022beyond}
{\sc Alghamdi, W., Hsu, H., Jeong, H., Wang, H., Michalak, P., Asoodeh, S., and Calmon, F.}
\newblock Beyond adult and compas: Fair multi-class prediction via information projection.
\newblock {\em Advances in Neural Information Processing Systems 35\/} (2022), 38747--38760.

\bibitem{amjad2019learning}
{\sc Amjad, R.~A., and Geiger, B.~C.}
\newblock Learning representations for neural network-based classification using the information bottleneck principle.
\newblock {\em IEEE transactions on pattern analysis and machine intelligence 42}, 9 (2019), 2225--2239.

\bibitem{atashin2021variational}
{\sc Atashin, A.~A., Razeghi, B., G{\"u}nd{\"u}z, D., and Voloshynovskiy, S.}
\newblock Variational leakage: The role of information complexity in privacy leakage.
\newblock In {\em Proceedings of the 3rd ACM Workshop on Wireless Security and Machine Learning\/} (2021), pp.~91--96.

\bibitem{bagdasaryan2019differential}
{\sc Bagdasaryan, E., Poursaeed, O., and Shmatikov, V.}
\newblock Differential privacy has disparate impact on model accuracy.
\newblock {\em Advances in neural information processing systems 32\/} (2019).

\bibitem{belghazi2018mutual}
{\sc Belghazi, M.~I., Baratin, A., Rajeshwar, S., Ozair, S., Bengio, Y., Courville, A., and Hjelm, D.}
\newblock Mutual information neural estimation.
\newblock In {\em International conference on machine learning\/} (2018), PMLR, pp.~531--540.

\bibitem{bertran2019adversarially}
{\sc Bertran, M., Martinez, N., Papadaki, A., Qiu, Q., Rodrigues, M., Reeves, G., and Sapiro, G.}
\newblock Adversarially learned representations for information obfuscation and inference.
\newblock In {\em International Conference on Machine Learning\/} (2019), PMLR, pp.~614--623.

\bibitem{calmon2017optimized}
{\sc Calmon, F., Wei, D., Vinzamuri, B., Natesan~Ramamurthy, K., and Varshney, K.~R.}
\newblock Optimized pre-processing for discrimination prevention.
\newblock {\em Advances in neural information processing systems 30\/} (2017).

\bibitem{calmon2015fundamental}
{\sc Calmon, F.~P., Makhdoumi, A., and M{\'e}dard, M.}
\newblock Fundamental limits of perfect privacy.
\newblock In {\em 2015 IEEE International Symposium on Information Theory (ISIT)\/} (2015), IEEE, pp.~1796--1800.

\bibitem{chang2021privacy}
{\sc Chang, H., and Shokri, R.}
\newblock On the privacy risks of algorithmic fairness.
\newblock In {\em 2021 IEEE European Symposium on Security and Privacy (EuroS\&P)\/} (2021), IEEE, pp.~292--303.

\bibitem{cheng2020club}
{\sc Cheng, P., Hao, W., Dai, S., Liu, J., Gan, Z., and Carin, L.}
\newblock Club: A contrastive log-ratio upper bound of mutual information.
\newblock In {\em International conference on machine learning\/} (2020), PMLR, pp.~1779--1788.

\bibitem{cormode2018privacy}
{\sc Cormode, G., Jha, S., Kulkarni, T., Li, N., Srivastava, D., and Wang, T.}
\newblock Privacy at scale: Local differential privacy in practice.
\newblock In {\em Proceedings of the 2018 International Conference on Management of Data\/} (2018), pp.~1655--1658.

\bibitem{creager2019flexibly}
{\sc Creager, E., Madras, D., Jacobsen, J.-H., Weis, M., Swersky, K., Pitassi, T., and Zemel, R.}
\newblock Flexibly fair representation learning by disentanglement.
\newblock In {\em International conference on machine learning\/} (2019), PMLR, pp.~1436--1445.

\bibitem{dai2018compressing}
{\sc Dai, B., Zhu, C., Guo, B., and Wipf, D.}
\newblock Compressing neural networks using the variational information bottleneck.
\newblock In {\em International Conference on Machine Learning\/} (2018), PMLR, pp.~1135--1144.

\bibitem{dieterich2016compas}
{\sc Dieterich, W., Mendoza, C., and Brennan, T.}
\newblock Compas risk scales: Demonstrating accuracy equity and predictive parity.
\newblock {\em Northpointe Inc 7}, 7.4 (2016), 1.

\bibitem{ding2017collecting}
{\sc Ding, B., Kulkarni, J., and Yekhanin, S.}
\newblock Collecting telemetry data privately.
\newblock {\em Advances in Neural Information Processing Systems 30\/} (2017).

\bibitem{du2012privacy}
{\sc du~Pin~Calmon, F., and Fawaz, N.}
\newblock Privacy against statistical inference.
\newblock In {\em 2012 50th annual Allerton conference on communication, control, and computing (Allerton)\/} (2012), IEEE, pp.~1401--1408.

\bibitem{Dua:2019}
{\sc Dua, D., and Graff, C.}
\newblock {UCI} machine learning repository, 2017.

\bibitem{dwork2011firm}
{\sc Dwork, C.}
\newblock A firm foundation for private data analysis.
\newblock {\em Communications of the ACM 54}, 1 (2011), 86--95.

\bibitem{dwork2012fairness}
{\sc Dwork, C., Hardt, M., Pitassi, T., Reingold, O., and Zemel, R.}
\newblock Fairness through awareness.
\newblock In {\em Proceedings of the 3rd innovations in theoretical computer science conference\/} (2012), pp.~214--226.

\bibitem{dwork2006calibrating}
{\sc Dwork, C., McSherry, F., Nissim, K., and Smith, A.}
\newblock Calibrating noise to sensitivity in private data analysis.
\newblock In {\em Theory of Cryptography: Third Theory of Cryptography Conference, TCC 2006, New York, NY, USA, March 4-7, 2006. Proceedings 3\/} (2006), Springer, pp.~265--284.

\bibitem{dwork2014algorithmic}
{\sc Dwork, C., Roth, A., et~al.}
\newblock The algorithmic foundations of differential privacy.
\newblock {\em Foundations and Trends in Theoretical Computer Science\/} (2014), 211--407.

\bibitem{erlingsson2014rappor}
{\sc Erlingsson, {\'U}., Pihur, V., and Korolova, A.}
\newblock Rappor: Randomized aggregatable privacy-preserving ordinal response.
\newblock In {\em Proceedings of the 2014 ACM SIGSAC conference on computer and communications security\/} (2014), pp.~1054--1067.

\bibitem{evfimievski2003limiting}
{\sc Evfimievski, A., Gehrke, J., and Srikant, R.}
\newblock Limiting privacy breaches in privacy preserving data mining.
\newblock In {\em Proceedings of the twenty-second ACM SIGMOD-SIGACT-SIGART symposium on Principles of database systems\/} (2003), pp.~211--222.

\bibitem{farrand2020neither}
{\sc Farrand, T., Mireshghallah, F., Singh, S., and Trask, A.}
\newblock Neither private nor fair: Impact of data imbalance on utility and fairness in differential privacy.
\newblock In {\em Proceedings of the 2020 workshop on privacy-preserving machine learning in practice\/} (2020), pp.~15--19.

\bibitem{feldman2015certifying}
{\sc Feldman, M., Friedler, S.~A., Moeller, J., Scheidegger, C., and Venkatasubramanian, S.}
\newblock Certifying and removing disparate impact.
\newblock In {\em proceedings of the 21th ACM SIGKDD international conference on knowledge discovery and data mining\/} (2015), pp.~259--268.

\bibitem{fioretto2022differential}
{\sc Fioretto, F., Tran, C., Van~Hentenryck, P., and Zhu, K.}
\newblock Differential privacy and fairness in decisions and learning tasks: A survey.
\newblock {\em arXiv preprint arXiv:2202.08187\/} (2022).

\bibitem{fischer2020conditional}
{\sc Fischer, I.}
\newblock The conditional entropy bottleneck.
\newblock {\em Entropy 22}, 9 (2020), 999.

\bibitem{friedman2010data}
{\sc Friedman, A., and Schuster, A.}
\newblock Data mining with differential privacy.
\newblock In {\em Proceedings of the 16th ACM SIGKDD international conference on Knowledge discovery and data mining\/} (2010), pp.~493--502.

\bibitem{hardt2016equality}
{\sc Hardt, M., Price, E., and Srebro, N.}
\newblock Equality of opportunity in supervised learning.
\newblock In {\em NIPS\/} (2016), pp.~3315--3323.

\bibitem{jagielski2019differentially}
{\sc Jagielski, M., Kearns, M., Mao, J., Oprea, A., Roth, A., Sharifi-Malvajerdi, S., and Ullman, J.}
\newblock Differentially private fair learning.
\newblock In {\em International Conference on Machine Learning\/} (2019), PMLR, pp.~3000--3008.

\bibitem{jeong2022fairness}
{\sc Jeong, H., Wang, H., and Calmon, F.~P.}
\newblock Fairness without imputation: A decision tree approach for fair prediction with missing values.
\newblock In {\em Proceedings of the AAAI Conference on Artificial Intelligence\/} (2022), vol.~36, pp.~9558--9566.

\bibitem{kamishima2011fairness}
{\sc Kamishima, T., Akaho, S., and Sakuma, J.}
\newblock Fairness-aware learning through regularization approach.
\newblock In {\em 2011 IEEE 11th International Conference on Data Mining Workshops\/} (2011), IEEE, pp.~643--650.

\bibitem{kingma2013auto}
{\sc Kingma, D.~P., and Welling, M.}
\newblock Auto-encoding variational bayes.
\newblock {\em arXiv preprint arXiv:1312.6114\/} (2013).

\bibitem{kleinberg2018algorithmic}
{\sc Kleinberg, J., Ludwig, J., Mullainathan, S., and Rambachan, A.}
\newblock Algorithmic fairness.
\newblock In {\em Aea papers and proceedings\/} (2018), vol.~108, pp.~22--27.

\bibitem{lecun1998gradient}
{\sc LeCun, Y., Bottou, L., Bengio, Y., and Haffner, P.}
\newblock Gradient-based learning applied to document recognition.
\newblock {\em Proceedings of the IEEE 86}, 11 (1998), 2278--2324.

\bibitem{li2017gradient}
{\sc Li, Y., and Turner, R.~E.}
\newblock Gradient estimators for implicit models.
\newblock {\em arXiv preprint arXiv:1705.07107\/} (2017).

\bibitem{louizos2015variational}
{\sc Louizos, C., Swersky, K., Li, Y., Welling, M., and Zemel, R.}
\newblock The variational fair autoencoder.
\newblock {\em arXiv preprint arXiv:1511.00830\/} (2015).

\bibitem{madras2018learning}
{\sc Madras, D., Creager, E., Pitassi, T., and Zemel, R.}
\newblock Learning adversarially fair and transferable representations.
\newblock In {\em International Conference on Machine Learning\/} (2018), PMLR, pp.~3384--3393.

\bibitem{mehrabi2021survey}
{\sc Mehrabi, N., Morstatter, F., Saxena, N., Lerman, K., and Galstyan, A.}
\newblock A survey on bias and fairness in machine learning.
\newblock {\em ACM Computing Surveys (CSUR) 54}, 6 (2021), 1--35.

\bibitem{mozannar2020fair}
{\sc Mozannar, H., Ohannessian, M., and Srebro, N.}
\newblock Fair learning with private demographic data.
\newblock In {\em International Conference on Machine Learning\/} (2020), PMLR, pp.~7066--7075.

\bibitem{nan2020variational}
{\sc Nan, L., and Tao, D.}
\newblock Variational approach for privacy funnel optimization on continuous data.
\newblock {\em Journal of Parallel and Distributed Computing 137\/} (2020), 17--25.

\bibitem{oord2018representation}
{\sc Oord, A. v.~d., Li, Y., and Vinyals, O.}
\newblock Representation learning with contrastive predictive coding.
\newblock {\em arXiv preprint arXiv:1807.03748\/} (2018).

\bibitem{poole2019variational}
{\sc Poole, B., Ozair, S., Van Den~Oord, A., Alemi, A., and Tucker, G.}
\newblock On variational bounds of mutual information.
\newblock In {\em International Conference on Machine Learning\/} (2019), PMLR, pp.~5171--5180.

\bibitem{pujol2020fair}
{\sc Pujol, D., McKenna, R., Kuppam, S., Hay, M., Machanavajjhala, A., and Miklau, G.}
\newblock Fair decision making using privacy-protected data.
\newblock In {\em Proceedings of the 2020 Conference on Fairness, Accountability, and Transparency\/} (2020), pp.~189--199.

\bibitem{rassouli2021perfect}
{\sc Rassouli, B., and G{\"u}nd{\"u}z, D.}
\newblock On perfect privacy.
\newblock {\em IEEE Journal on Selected Areas in Information Theory 2}, 1 (2021), 177--191.

\bibitem{razeghi2020perfect}
{\sc Razeghi, B., Calmon, F.~P., G{\"u}nd{\"u}z, D., and Voloshynovskiy, S.}
\newblock On perfect obfuscation: Local information geometry analysis.
\newblock In {\em 2020 IEEE International Workshop on Information Forensics and Security (WIFS)\/} (2020), IEEE, pp.~1--6.

\bibitem{rodriguez2021variational}
{\sc Rodr{\'\i}guez-G{\'a}lvez, B., Thobaben, R., and Skoglund, M.}
\newblock A variational approach to privacy and fairness.
\newblock In {\em 2021 IEEE Information Theory Workshop (ITW)\/} (2021), IEEE, pp.~1--6.

\bibitem{rosenberg2023fairness}
{\sc Rosenberg, H., Tang, B., Fawaz, K., and Jha, S.}
\newblock Fairness properties of face recognition and obfuscation systems.
\newblock In {\em 32nd USENIX Security Symposium (USENIX Security 23)\/} (2023), pp.~7231--7248.

\bibitem{sarhan2020fairness}
{\sc Sarhan, M.~H., Navab, N., Eslami, A., and Albarqouni, S.}
\newblock Fairness by learning orthogonal disentangled representations.
\newblock In {\em Computer Vision--ECCV 2020: 16th European Conference, Glasgow, UK, August 23--28, 2020, Proceedings, Part XXIX 16\/} (2020), Springer, pp.~746--761.

\bibitem{strouse2017deterministic}
{\sc Strouse, D., and Schwab, D.~J.}
\newblock The deterministic information bottleneck.
\newblock {\em Neural computation 29}, 6 (2017), 1611--1630.

\bibitem{tishby2000information}
{\sc Tishby, N., Pereira, F.~C., and Bialek, W.}
\newblock The information bottleneck method.
\newblock {\em arXiv preprint physics/0004057\/} (2000).

\bibitem{tran2021differentially}
{\sc Tran, C., Fioretto, F., and Van~Hentenryck, P.}
\newblock Differentially private and fair deep learning: A lagrangian dual approach.
\newblock In {\em Proceedings of the AAAI Conference on Artificial Intelligence\/} (2021), vol.~35, pp.~9932--9939.

\bibitem{van2017neural}
{\sc Van Den~Oord, A., Vinyals, O., et~al.}
\newblock Neural discrete representation learning.
\newblock {\em Advances in neural information processing systems 30\/} (2017).

\bibitem{veale2021demystifying}
{\sc Veale, M., and Zuiderveen~Borgesius, F.}
\newblock Demystifying the draft eu artificial intelligence act—analysing the good, the bad, and the unclear elements of the proposed approach.
\newblock {\em Computer Law Review International 22}, 4 (2021), 97--112.

\bibitem{voigt2017eu}
{\sc Voigt, P., and Von~dem Bussche, A.}
\newblock The eu general data protection regulation (gdpr).
\newblock {\em A Practical Guide, 1st Ed., Cham: Springer International Publishing 10}, 3152676 (2017), 10--5555.

\bibitem{warner1965randomized}
{\sc Warner, S.~L.}
\newblock Randomized response: A survey technique for eliminating evasive answer bias.
\newblock {\em Journal of the American Statistical Association 60}, 309 (1965), 63--69.

\bibitem{wei2020federated}
{\sc Wei, K., Li, J., Ding, M., Ma, C., Yang, H.~H., Farokhi, F., Jin, S., Quek, T.~Q., and Poor, H.~V.}
\newblock Federated learning with differential privacy: Algorithms and performance analysis.
\newblock {\em IEEE Transactions on Information Forensics and Security 15\/} (2020), 3454--3469.

\bibitem{wen2020mutual}
{\sc Wen, L., Zhou, Y., He, L., Zhou, M., and Xu, Z.}
\newblock Mutual information gradient estimation for representation learning.
\newblock {\em arXiv preprint arXiv:2005.01123\/} (2020).

\bibitem{xie2022robust}
{\sc Xie, S., Wu, Y., Ma, S., Ding, M., Shi, Y., and Tang, M.}
\newblock Robust information bottleneck for task-oriented communication with digital modulation.
\newblock {\em arXiv preprint arXiv:2209.10382\/} (2022).

\bibitem{zafar2017fairness}
{\sc Zafar, M.~B., Valera, I., Rogriguez, M.~G., and Gummadi, K.~P.}
\newblock Fairness constraints: Mechanisms for fair classification.
\newblock In {\em Artificial intelligence and statistics\/} (2017), PMLR, pp.~962--970.

\bibitem{zemel2013learning}
{\sc Zemel, R., Wu, Y., Swersky, K., Pitassi, T., and Dwork, C.}
\newblock Learning fair representations.
\newblock In {\em International conference on machine learning\/} (2013), PMLR, pp.~325--333.

\bibitem{zhao2019conditional}
{\sc Zhao, H., Coston, A., Adel, T., and Gordon, G.~J.}
\newblock Conditional learning of fair representations.
\newblock {\em arXiv preprint arXiv:1910.07162\/} (2019).

\end{thebibliography}

\appendix
\section{Appendix}
\subsection{Proofs of the variational encoding}\label{app: variational}
In this section, we proceed to present the proofs of the implementation in Section~\ref{sec: variational}. We present the detailed derivation of variational lower bounds of the proposed objective $\mathcal{L}(\bphi)$ and prove the proposed randomizers for intermediate representations, Laplace mechanisms and randomized response, satisfy $\epsilon$-LDP.

\subsubsection{Derivation of the variational lower bound}\label{app: variational-deriv}
The Lagrangian of the proposed optimization problem is $\mathcal{L}(\bphi) = I(X;Z|S) + \beta I(U;Z)$. We write it in full expression and derive the variational lower bound presented in \eqref{eq: variational}, as shown in the following.
\begin{figure*}[t]
\begin{equation}\label{eq:proof}
\scriptstyle
\begin{aligned}
    \cL (\phi)&=  I(X; Z|S) + \beta I(U; Z) \\
    & = H(X|S) - H(X|Z;S) + \beta [H(U)-H(U|Z)]\\
    & = \bE_{p(s, u, x)} \{ \bE_{p_{\bphi}(z|x)}[ \log p_{\bphi}(x|z, s) ] + \beta \bE_{p_{\bphi}(z|x)} [ \log p_{\bphi}(u|z)]\} + H(X|S) +\beta H(U)\\
    & = \underbrace{\bE_{p(s, u, x)} \{ \bE_{p_{\bphi}(z|x)}[ \log q_{\btheta}(x|z, s) ] + \beta \bE_{p_{\bphi}(z|x)} [ \log q_{\btheta}(u|z)]\}}_{\mathcal{L}_V(\bphi, \btheta)} + \underbrace{H(X|S) +\beta H(U)}_{\text{const}} \\
    &+\bE_{p(z, s)} \{ \underbrace{\bE_{p_{\bphi}(x|z, s)}[ \log \frac{p_{\bphi}(x|z, s)}{q_{\btheta}(x|z, s)} ]}_{KL(p_{\bphi}(x|z, s)\|q_{\btheta}(x|z, s))\geq 0\}} + \beta \bE_{p(z)} \{\underbrace{\bE_{p_{\bphi}(u|z)} [ \log \frac{p_{\bphi}(u|z)}{q_{\btheta}(u|z)} ]}_{KL(p_{\bphi}(u|z)\|q_{\btheta}(u|z))\geq 0} \} \\
    &\geq \mathcal{L}_V(\bphi, \btheta)
\end{aligned}
\end{equation}
\hrule
\end{figure*}
where the KL-divergence is non-negative and the terms $H(X|S)$ and $H(U)$ are constant given a data source $P_{USX}$. Therefore, the variational lower bound of $\mathcal{L}(\bphi)$ is $\mathcal{L}_V(\bphi,\btheta)$. 

\subsubsection{Continuous Representation with Laplace mechanism} \label{app: LM}
When $\hat{z}$ is $d$-dimensional continuous representation $\hat{z} = (\hat{z}_1, \hat{z}_2, \dots, \hat{z}_d)\in \bR^{d}$ that has been subjected to a local truncation with a threshold $t$ for each dimension, a Laplace mechanism is $\epsilon$-LDP by adding artificial Laplace noises,
\begin{align}
    z = \hat{z} + (n_1, n_2, \dots, n_d),
\end{align}
for each $\hat{z}$, $n_i \sim \text{Laplace}(0, 2td/\epsilon)$.
\begin{proof}
    Using the local truncation with a threshold $t$ on the intermediate representations, we can limit the sensitivity for each pair of representations $\hat{z}' \not = \hat{z} \in \bR^d$ by
    \begin{align}
        \|\hat{z}' -\hat{z}\|_1 \leq 2td.\label{eq: truncation}
    \end{align}
    Then, the probability ratio for each pair $\hat{z}' \not = \hat{z} \in \bR^d$ with Laplace noise $\text{Laplace}(0, 2td/\epsilon)$, 
    \begin{align}
        \frac{\Pr[Z=z|\hat{z}]}{\Pr[Z=z|\hat{z}']} & = \prod\limits_{i=1}^d \frac{\Pr[Z_i=z_i|\hat{z}_i]}{\Pr[Z_i=z_i|\hat{z}_i']}  \\
        & = \prod\limits_{i=1}^d \frac{\exp(-\frac{\epsilon|z_i-\hat{z}_i|}{2td})}{\exp(-\frac{\epsilon|z_i-\hat{z}'_i|}{2td})}\\
        & = \prod_{i=1}^d\exp(\frac{\epsilon|z_i - \hat{z}'_i|-\epsilon|z_i- \hat{z}_i|}{2td})      \\
        & \overset{(a)}{\leq} \prod_{i=1}^d\exp(\frac{\epsilon|\hat{z}_i-\hat{z}'_i|}{2td}) \\
        & = \exp(\frac{\epsilon\|\hat{z}' -\hat{z} \|_1}{2td})\\
        & \overset{(b)}{\leq} \exp(\epsilon),
    \end{align}
    where the inequality $(a)$ follows from the triangle inequality and $(b)$ follows from the result of local truncation \eqref{eq: truncation}. The result of $ \frac{\Pr[Z=z|\hat{z}]}{\Pr[Z=z|\hat{z}']} \geq \exp(-\epsilon)$ can be easily obtained by symmetry. 
\end{proof}

\subsubsection{Discrete representation with randomized response}\label{app: RR}
Given an encoder with the output representation $\hat{z} = (\hat{z}_1, \hat{z}_2, \dots, \hat{z}_d) \in [k]^d $, the $\epsilon$-LDP randomized response $\cM: \mathcal{\hat{Z}} \to \mathcal{Z}$ works as follow, for any $i \in [d]$ and $u\in [k]$, 
\begin{align}
    \Pr[z_i = v|u] = \left\{  
\begin{array}{lr}
    \frac{e^{\epsilon/d}}{e^{\epsilon/d}+k-1}, & \text{if } v = u,    \\
    \frac{1}{e^{\epsilon/d}+k-1} , &\text{if } v \not = u. 
\end{array}
\right. 
\end{align} 
\begin{proof}
For any $\hat{z}\not = \hat{z}' \in [k]^d$, the probability ratio is bounded by the worst case of released representation, $v \not =  u$ for each $i \in[d]$. Specifically, 
\begin{align}
    \log \frac{\Pr[Z = z|\hat{z}]}{\Pr[Z = z|\hat{z}']} & \overset{(a)}{\leq}  \sum\limits_{i=1}^{d} \log \frac{\Pr[Z_i=u| u]}{\Pr[Z_i=u|v]}\\
    & = d \log \frac{ \frac{e^{\epsilon/d}}{e^{\epsilon/d}+k-1}}{\frac{1}{e^{\epsilon/d}+k-1}}\\
    & = d \log e^{\epsilon/d}\\
    & = \epsilon,
\end{align}
where the inequality $(a)$ follows from $e^{\epsilon/d} \geq 1$ for any $\epsilon \geq 0$. The result of $\log \frac{\Pr[Z = \hat{z}|\hat{z}]}{\Pr[Z = z|\hat{z}]} \geq -\epsilon$ follows by symmetry.
\end{proof}
\subsubsection{Discrete representation encoding} \label{app: VQVAE}
We introduce a vector-quantization module to encode the discrete representations following the seminal work of VQ-VAE~\cite{van2017neural}. For a discrete representation $z \in [k]$, we define a set of embedding vectors $E \in \bR^{K\times D}$ and discretize the encoding features $\text{enc}(x) \in \bR^D$ using a nearest-neighbor lookup on $E$. Specifically, the discrete value $k$ is obtained by
\begin{align}
    k  = \argmin\limits_{j \in [K]} \| \text{enc}(x) - E_j\|_2,
\end{align}
And the decoders take the input of $z_q(x)$ as the $k$-th embedding, $z_q(x) = E_k$.    
Then, the model is trained to minimize the total loss function $\mathcal{L}_{\text{Dis}}$ that combines the proposed loss $\mathcal{L}_V$ and the vector-quantization terms. Thus, the total training objective for discrete encoding is:
\begin{align}
    \mathcal{L}_{\text{Dis}} & = \mathcal{L}_V + \|\text{sg}[\text{enc}(x)] - z_q(x) \|_2^2 + \lambda \|\text{enc}(x) - \text{sg}[z_q(x)] \|_2^2,
\end{align}
where sg is the stopgradient operator. We set $\lambda = 0.25$ for all the experiments, following the seminal work~\cite{van2017neural}. 
\subsection{Experiment Details}
\subsubsection{Neural Network Architectures}

\paragraph{Colored-MNIST experiment:} The neural network architectures of our method for the Colored-MNIST dataset are described as follows.

\begin{itemize}
    \item \textbf{Encoder:} The encoder is a convolutional neural network with three convolutional layers and ReLU activations.
    \item \textbf{Utility decoder:} The utility decoder adopts a multi-layer perceptron with two 100-unit hidden layers and ReLU activations. When the utility variable $U$ represents color, the output is 3-dimensional with a Sigmoid activation function. When the utility variable $U$ represents a digit number, the output is 10-dimensional with a Sigmoid activation function.
    \item \textbf{Side decoder:} The side decoder neural network is shown in Table~\ref{tab: networks}. Since the side decoder takes representation $Z$ and sensitive variable $S$ as input with different dimensions, it is essential to apply appropriate fusion operations. First, we perform one-hot encoding on the sensitive variable $S$ so that it has channels equal to the cardinality of $S$, $|S|$. Then the size in each channel is set to $10\times 10$ by duplicating the label 0 or 1. Therefore, the sensitive variable $S$ is represented by a tensor with the size of $|S|\times 10 \times 10$. Subsequently, the processed $S$ and $Z$ are applied with a convolutional layer separately after which $Z$ and $S$ can be concatenated along the channel dimension to go through the following layers.
\end{itemize}
   \begin{table}[t]
		\centering
     		\caption{The side decoder architecture for the experiments on \texttt{colored MNIST} dataset. Conv2D (cin, court, ksize, stride, padding) represents a 2D convolution layer, where cin is the number of input channels, cout is the number of output channels, ksize is the size of the kernel. BN, ReLU, ReLU6, and Sigmoid represent batch normalization, ReLU, ReLU6, and Sigmoid layers, respectively.}

		\label{tab: networks}

			\begin{tabular}{cl}
				\toprule
                Modules & Architecture\\
    \midrule
 \multirow{6}{*}{Side decoder} & Conv2d (3, 10, 5, 1, 2), BN, ReLU\\
 &Conv2d ($|S|$, 10, 5, 1, 2), BN, ReLU\\
 &Concatenate\\
 &Conv2d (20, 100, 5, 1, 2), BN, ReLU\\
 &Conv2d (100, 50, 5, 1, 2), BN, ReLU6\\
 &Conv2d (50, 3, 5, 1, 2), Sigmoid \\
				\bottomrule
		\end{tabular}
	\end{table}

\paragraph{Adult and compas Experiment:} Since the \texttt{Adult} dataset and \texttt{compas} dataset have similar datum form, the network architectures are almost the same except for the input dimension. The representation $Z$ is 2 dimensional with a truncation on each dimension with threshold $t=0.5$.
\begin{itemize}
    \item \textbf{Enocder:} The encoder adopts a multi-layer perceptron with a single hidden layer with 100 units and ReLU activations.
    \item \textbf{Utility decoder:} The utility decoder adopts multi-layer perceptron with two 100-units hidden layers and ReLU activations. Since the utility variable $U$ is binary, the output is 2-dimensional with a Sigmoid activation function.
    \textbf{Side decoder:} The side decoder is a multi-layer perceptron with a single hidden layer with 100 units and ReLU activations. As both the representation $Z$ and sensitive variable $S$ are low-dimensional vectors, a concatenation of $Z$ and $S$ is directly used as input to the network.
\end{itemize}
   \begin{table*}
		\centering
 		\caption{Hyperparameters used in the experiments}
		\label{tab: hyperparameters}
		\resizebox{.95\linewidth}{!}{
			\begin{tabular}{cccccc}
				\toprule
                       Dataset & Methods & Epoches & Batchsize & Learning rate & Lagrange multiplier $\beta$\\
				\midrule
    \texttt{colored MNIST}& Con & 150& 256  & $10^{-3}$ & $[10^{-3}, 10]$ (logarithmically spaced) \\
    \texttt{compas}& Con & 100& 64  & $10^{-3}$ & $[10^{-3}, 10^3]$ (logarithmically spaced) \\
    \texttt{compas}& Dis & 100& 64  & $10^{-3}$ & $[10^{-3}, 10^3]$ (logarithmically spaced) \\
    \texttt{adult}& Con & 150& 512  & $10^{-3}$ & $[10^{-3}, 10^3]$ (logarithmically spaced) \\
    \texttt{adult}& Dis & 150& 128  & $10^{-3}$ & $[10^{-3}, 10^3]$ (logarithmically spaced) \\

				\bottomrule
		\end{tabular}}
	\end{table*}

\subsubsection{Baseline methods}
\begin{itemize}
    \item CFAIR (Conditional fair representations): CFAIR is a fair representations learning method that can simultaneously mitigate two notions of disparity among different demographic subgroups in the classiﬁcation setting, based on two key components balanced error rate and conditional alignment of representations~\cite{zhao2019conditional}. 
    \item PPVAE (Privacy Preserving Variational Autoencoder): PPVAE is a variational approach that does not rely on adversarial training under the setting that the private variable are high-dimensional and continuous data. We follow the settings in ~\cite{rodriguez2021variational}, which expands the original method for the non-continuous data by assuming categorical and Bernoulli output distributions~\cite{nan2020variational}.
    \item FFVAE (Flexibly Fair VAE): FFVAE is a method for learning compact representations that are useful for reconstruction and prediction, but are also ﬂexibly fair which means they can be easily modiﬁed at test time to achieve subgroup demographic parity with respect to multiple sensitive attributes~\cite{creager2019flexibly}. 
    \item CFB (Conditional Fairness Bottleneck): CFB is a variational approach to learn fair representations that keep a prescribed level of the relevant information which is not shared by the sensitive data while minimizing the remaining information they keep~\cite{rodriguez2021variational}. 
    \item VFAE (Variational Fair Autoencoder): VFAE is a model based on deep variational autoencoders with priors that encourage independence between sensitive and utility variables, as well as subsequent tasks such as classification performed on the representations~\cite{louizos2015variational}.
\end{itemize}

\subsubsection{Evaluation Metrics}
\paragraph{MI estimation:} We estimate MI $I(U;Z)$ and $I(S;Z)$ using the mutual information neural estimator (MINE)~\cite{belghazi2018mutual}. The neural network adopts a 2-hidden-layer multi-layer perceptron with both 100 hidden units and ReLU6 activation functions. The network takes 50000 iterations to approximate the estimator with an exponential moving average rate of 0.01, and the result value is estimated by averaging over the last 100 iterations. The hyperparameters used to train the networks are displayed in Table~\ref{tab: hyperparameters}.

   \begin{table}
		\centering
 		\caption{Hyperparameters used in MINE computation}
		\label{tab: MINE}
		
			\begin{tabular}{cccc}
				\toprule
                       Dataset  & Iterations & Batchsize & Learning rate   \\
				\midrule
    \texttt{compas}& $5\times 10^{4}$& 463  & $10^{-3}$  \\
    \texttt{adult} & $5\times 10^{4}$& 1024  & $10^{-3}$ \\

				\bottomrule
		\end{tabular}
	\end{table}
 
\begin{definition}[Demographic Parity Gap~\cite{zemel2013learning}]
The demographic parity gap or discrimination of a predictor $\hat{U}$ with sensitive variable $S$ is defined as
    \begin{align}
        \Delta_{\mathrm{DP}}\triangleq{\mid\Pr[\hat{U}=1 | S=0]}-\Pr[\hat{U}=1 | S=1]\mid. \label{eq: dpg}  
    \end{align}
\end{definition}    
A predictor $\hat{U}$ satisfies \emph{demographic parity} if $\Delta_{\text{DP}} = 0$ holds. In other words, the ratio of positive outcomes of the target should be the same across different protected groups, which implies the prediction is independent of $S$.Therefore, it is also known as \emph{statistical parity} and has been widely used for fairness evaluation. Notably, demographic parity may reduce the utility when the desired prediction of the utility variable is significantly different among each sensitive group ~\cite{dwork2012fairness}.

\begin{definition}[Equalized Odds Gap~\cite{dwork2012fairness}]
The equalized odds gap of a predictor $\hat{U}$ with sensitive variable S and utility variable U is defined as:
    \begin{align}
        \Delta_{\mathrm{EO}}\triangleq\max\limits_{u \in \left\{0,1\right\}}&|\Pr[\hat{U}=1 | S=0,U=u]\\
        &-\Pr[\hat{U}=1 | S=1,U=u]|. \label{eq: eo}  
    \end{align}
\end{definition}  
A predictor $\hat{U}$ satisfies \emph{equalized odds} if $\Delta_{\text{EO}} = 0$ holds, which implies the prediction is independent of $S$ conditioning on $U$. In other words, the ratio of true positives outcomes and false positives outcomes should both be the same for different protected groups.Therefore, it is also known as \emph{positive rate parity}. Unlike demographic parity, equalized odds does not conflict with a predictor with perfect performance for utility
 since a completely accurate predictor satisfying $\hat{U} = U$ conforms with the definition of positive rate parity~\cite{hardt2016equality}.

\subsubsection{Evaluation procedure}
The evaluation procedures in the experiments are described as follows. 
\begin{enumerate}
\item Firstly, all the evaluated methods are trained on the training set until convergence, where the dimensions of representation $Z$ are uniformly set to $d=2$. 
\item Then, we adopt 1-hidden-layer fully-connected neural networks as the downstream inference models and train them on the representation obtained from each evaluated method to predict the utility $U$. For the proposed method, we take the utility decoder as the downstream inference model as it has the same architecture and is designed for the downstream task. Furthermore, we also train the downstream inference models on the representations obtained by the baselines with Laplace mechanisms to investigate the impact of LDP mechanisms on the baseline methods.
\item Finally, the trained inference models are used to perform the prediction of utility $U$ on the test set.
The task accuracy and fairness metrics are computed based on the prediction, and sensitive information leakage $I(U;Z)$ is estimated using MINE. 
\end{enumerate}
\end{document}